\documentclass[10pt,journal,compsoc]{IEEEtran}
\ifCLASSOPTIONcompsoc
  \usepackage[nocompress]{cite}
\else
  \usepackage{cite}
\fi
\ifCLASSINFOpdf
\else
\fi
\usepackage[utf8]{inputenc} 
\usepackage[T1]{fontenc}    
\usepackage{hyperref}       
\usepackage{url}            
\usepackage{booktabs}       
\usepackage{amsfonts}       
\usepackage{nicefrac}       
\usepackage{microtype}      

\usepackage{times}
\usepackage{epsfig}
\usepackage{epstopdf}

\usepackage{graphicx}
\usepackage{amsmath}
\usepackage{amssymb}
\usepackage{bm} 
\usepackage{calrsfs} 
\usepackage{xcolor}
\usepackage{scalerel}
\usepackage{booktabs} 
\usepackage[ruled]{algorithm2e} 
\usepackage{subcaption}
\usepackage[colorinlistoftodos,prependcaption,textsize=tiny]{todonotes}
\usepackage{physics} 
\usepackage{amsthm} 
\usepackage{multirow}

\def\tp{{}^{\top}}

\DeclareMathAlphabet{\calli}{OMS}{zplm}{m}{n}

\DeclareMathOperator{\sspan}{span}
\DeclareMathOperator{\orth}{orth}
\DeclareMathOperator{\Exp}{Exp}

\newtheorem*{proposition}{Proposition}

\begin{document}
%
\title{Grassmannian learning mutual subspace method for image set recognition}

\author{Lincon~S.~Souza,
        Naoya~Sogi,
        Bernardo~B.~Gatto,
        Takumi~Kobayashi,
        and~Kazuhiro~Fukui 
\IEEEcompsocitemizethanks{
\IEEEcompsocthanksitem L.S. Souza, T. Kobayashi and B.B. Gatto are with the National Institute of Advanced Industrial Science and Technology (AIST), Japan.
\IEEEcompsocthanksitem K. Fukui and N. Sogi are with the Department of Computer Science, Graduate School of Systems and Information Engineering, University of Tsukuba, Japan. (E-mail: lincon.souza@aist.go.jp)}
}


\IEEEtitleabstractindextext{%
\begin{abstract}
This paper addresses the problem of object recognition given a set of images as input (e.g., multiple camera sources and video frames). Convolutional neural network (CNN)-based frameworks do not exploit these sets effectively, processing a pattern as observed, not capturing the underlying feature distribution as it does not consider the variance of images in the set. To address this issue, we propose the Grassmannian learning mutual subspace method (G-LMSM), a NN layer embedded on top of CNNs as a classifier, that can process image sets more effectively and can be trained in an end-to-end manner. The image set is represented by a low-dimensional input subspace; and this input subspace is matched with reference subspaces by a similarity of their canonical angles, an interpretable and easy to compute metric. The key idea of G-LMSM is that the reference subspaces are learned as points on the Grassmann manifold, optimized with Riemannian stochastic gradient descent. This learning is stable, efficient and theoretically well-grounded. We demonstrate the effectiveness of our proposed method on hand shape recognition, face identification, and facial emotion recognition.
\end{abstract}

}

\maketitle

\IEEEdisplaynontitleabstractindextext


\IEEEraisesectionheading{\section{Introduction}\label{sec:introduction}}

\begin{figure*}[t]
\centering
\includegraphics[width=0.9\linewidth]{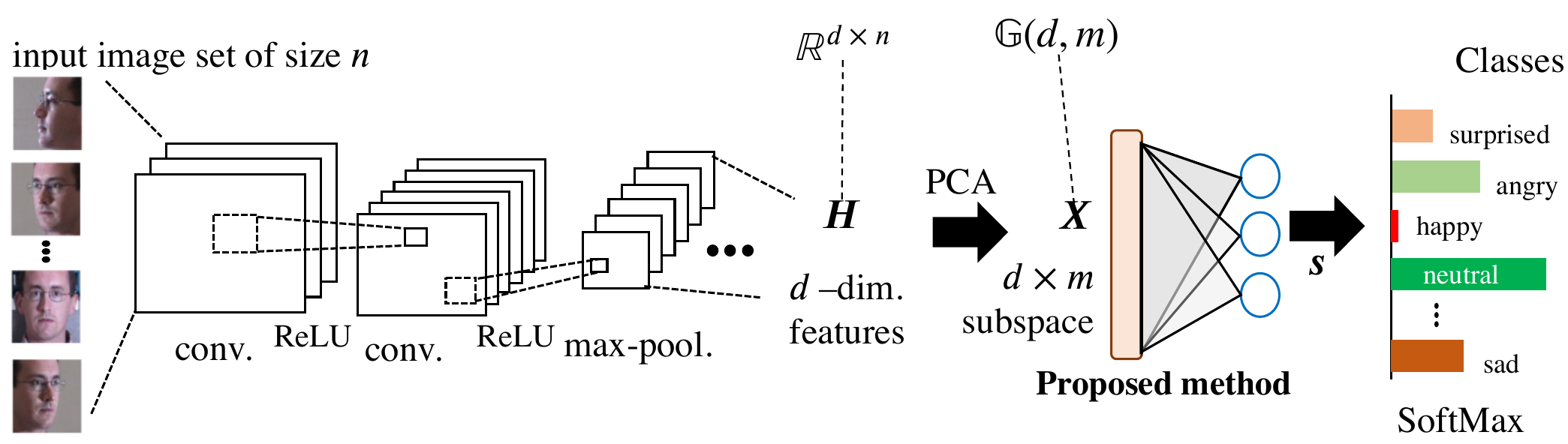}
\caption{Diagram of the proposed method embedded as a layer in an end-to-end learning framework.}
\label{fig:diagram}
\end{figure*}

\begin{figure*}[t]
\centering
\includegraphics[width=\linewidth]{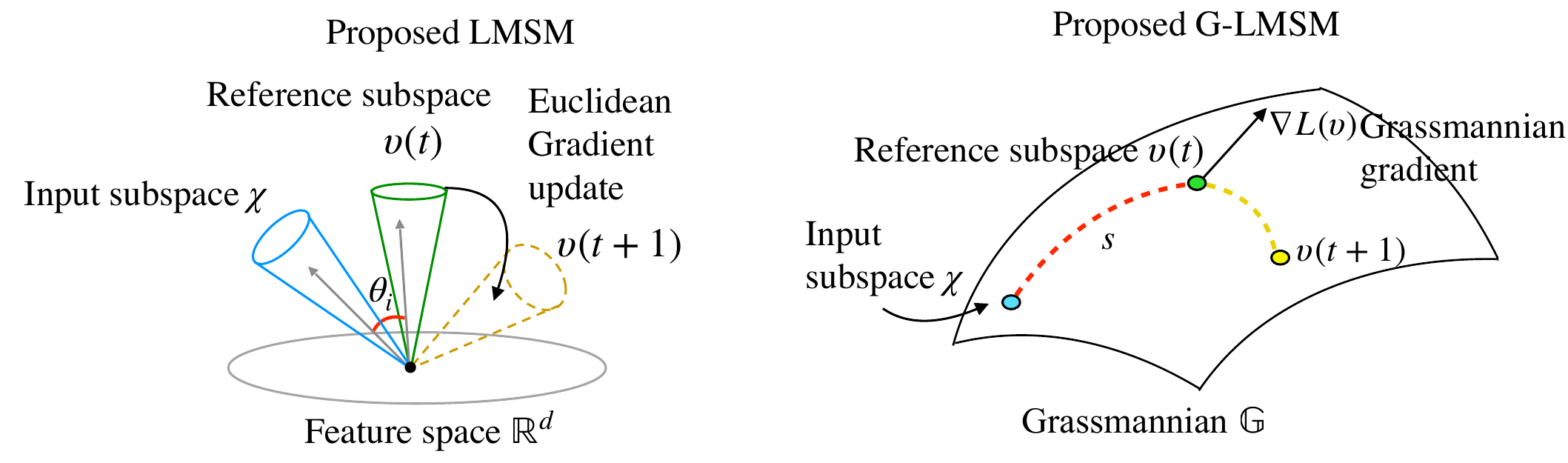}
\caption{Conceptualized depiction of the proposed LMSM and G-LMSM key idea. Both proposed methods match an input subspace $\bm \chi $ (blue) to a reference subspace $\bm{\upsilon}_j(t)$ (green) by using a similarity $s_j$ (red) based on their canonical angles. In the learning step, LMSM updates the reference subspaces to $\bm{\upsilon}_j(t+1)$ (yellow) using a gradient vector; LMSM uses a gradient on the Euclidean matrix space, and G-LMSM uses the gradient $\nabla L(\bm \upsilon_j)$ (black) on the Grassmannian $\mathbb{G}$.}

\label{fig:keyidea}
\end{figure*}

\IEEEPARstart{M}{ultiple} images of an object are useful for boosting performance of object classification~\cite{shashua_photometric_1997, belhumeur_what_1998}. In surveillance and industrial systems, distributed cameras and sensors capture data in a continuous stream to provide multiple images pointing to a target object. The task of object recognition based on a set of images, called \textit{image set recognition}, is formulated to classify an image set through comparing the input set to reference image sets (dictionary) of class objects.

The image set that we consider in this paper refers to a collection of image samples captured from the same object or event, thereby freeing images in a set from any order, such as temporal ordering. For example, image sets of our interest include multi-view images from multiple camera sources and video frames. As in the other image classification tasks, each image in the set could be represented by a naive image vector~\cite{lee2005acquiring, basri2003lambertian, kim2007discriminative} or by the more sophisticated feature vector via CNNs~\cite{sogi2022constrained, souza2020enhanced}, to provide us with a set of feature \textit{vectors} for set classification. Therefore, the primal issue to be addressed in the set recognition is how to represent the \textit{set} which is composed of variable number of feature vectors.

Subspace-based methods~\cite{watanabe1973subspace, iijima1974theory, oja1983subspace} have been one of the leading solutions for image set recognition, and a central research topic for object recognition in computer vision. In the subspace methods, a set of images is effectively modeled by a low-dimensional \textit{subspace} in a high-dimensional vector space~\cite{georghiades2001few, shashua_photometric_1997, belhumeur_what_1998, basri2003lambertian}. Such subspaces are typically generated by applying either singular value decomposition (SVD) or Gram–Schmidt orthogonalization to feature vectors in a set. The subspace-based set representation is effective in the following points. (1) A subspace exploits the geometrical structure of the probability distribution underlying an image set. (2) The subspaces are also statistically robust to input noises, e.g., occlusion in an image. (3) In addition, a subspace is a compact representation of a set; we can describe a set by fixed and low dimensional subspace no matter how many image frames are included in a set.

In the subspace-based methods, for classifying sets into class categories, an input subspace is compared to \textit{reference subspaces} by using a structural similarity based on \textit{canonical angles}~\cite{hotelling1992relations, afriat1957orthogonal} between subspaces. For example, suppose we have image sets of 3D objects, and it can be observed that under some assumptions, the structural similarity is related to 3D shape similarity~\cite{belhumeur_what_1998, lee2005acquiring, basri2003lambertian}, which enables us to implicitly compare object shapes by means of subspaces. Therefore, to further improve performance of subspace-based set classification, there are two directions regarding optimization of feature space on which the structural similarity is measured and that of reference subspace to describe dictionary samples.

For improving the feature space, some methods are proposed in the subspace literature, such as constrained mutual subspace method (CMSM)~\cite{fukui2015difference, yamaguchi2003smartface} and orthogonal MSM~\cite{fukui2007kernel}, by exploring \textit{linear projection} from the original feature space to the space that is contributive to discriminatively measuring the structural similarities. Though the linear projection is extended to non-linear one via kernel tricks~\cite{sakano2000kernel, hamm2008grassmann, hamm2009extended}, the space that those methods produce is highly dependent on the input feature space in which image frames are originally represented. As to optimization of reference subspaces, the seminal work is performed by Kohonen and Oja~\cite{kohonen1979spectral, oja1983subspace} to formulate the learning subspace methods (LSM). In the LSM~\cite{kohonen1979spectral}, a class reference subspace is iteratively improved by means of rotation based on the statistical decision-theoretic rule, beyond the simple SVD which computes reference subspace in a generative way. LSM is extended to averaged LSM (ALSM)~\cite{oja_alsm_1983} toward learning stability by considering iterative updating based on \textit{mini-batch}. These two research directions are made in the subspace literature rather separately, while they could complement each other.

In this paper, we formulate a new set-classification framework (Fig.~\ref{fig:diagram}) to learn \textit{both} the reference subspaces and the feature space (representation) in which the set subspaces are extracted. It contains a key module which we propose to learn reference subspaces and is conceptualized in Fig.~\ref{fig:keyidea}. As in previous mutual subspace method, an input subspace is classified based on matching to class reference subspace by using canonical angle, a structural similarity. It should be noted that the reference subspaces are learned by means of the gradient descent and thereby it is possible to embed the learning subspace functionality into an end-to-end learning scheme to simultaneously learn image feature representation beyond the naive linear projection of off-the-shelf features~\cite{fukui2015difference, hamm2009extended}. 

From the viewpoint of learning subspace, the proposed method is connected to Oja’s LSM. There, however, are several key differences between those two methods. The LSM improves a reference subspace in an instance-based manner by considering vector-to-subspace classification, while our method learns subspace based on subspace-to-subspace comparison for set classification. Besides, in contrast to the heuristic updating rule in LSM, we apply gradient-based update to optimize the reference subspaces in a theoretical way.

Vector-to-subspace comparison in LSM is simply defined in Euclidean feature space. On the other hand, subspace-to-subspace relationships are naturally formulated in Grassmann manifold~\cite{absil2004riemannian, borisenko1991grassmann, fujii2002introduction}, a non-linear Riemannian manifold beyond the Euclidean space, thereby imposing difficulty on the gradient-based updating. While the gradient-based optimization has been usually applied in Euclidean parameter space such as by SGD, we establish the subspace updating formula by applying Riemannian stochastic gradient descent (RSGD)~\cite{ganea2018riemannian, bonnabel2013stochastic} to naturally optimize a subspace along a geodesic path towards (local) minima. It endows the learning subspace with an elegant geometrical structure of subspace as well as a clear and strong theoretical background and effective optimization; our method directly updates subspace representation without ad-hoc post-processing such as SVD to ensure orthonormality of the subspace basis which is time-consuming and degrades optimality of the updating. 

The proposed method to learn subspaces can thus be plugged into an end-to-end learning framework which optimizes neural networks through back-propagation, as shown in Fig.~\ref{fig:diagram}. As a result, the method contributes to learn both reference subspaces and feature representation in which the subspaces are extracted. The proposed set classification is formulated in such a general form that (1) it can cope with pre-designed (off-the-shelf) features to only learn reference subspaces and (2) can be embedded in an \textit{intermediate} layer to represent an input set by a fixed-dimensional vector as set representation learning; from that viewpoint, the above-mentioned set classification (Fig.~\ref{fig:diagram}) is regarded as a set-based classification layer whose counterpart in a standard CNN is implemented in fully-connected (FC) layer.

Our contributions are summarized as follows.
\begin{itemize}
    \item We propose two new subspace-based methods named learning mutual subspace method (LMSM) and Grassmannian learning mutual subspace method (G-LMSM). They match subspaces and learns reference subspaces: LMSM learns on the Euclidean space, whereas G-LMSM learns on the Grassmannian.
  	\item We combine LMSM and G-LMSM with CNNs, by employing CNN as backbones and our methods as a classifier, yielding a straightforward yet powerful model for image set recognition.
  	We showcase this combination by training models in an end-to-end manner and applying them to the problems of hand shape, face recognition, and face expression recognition from video data.
  	\item We also propose variations of the : (1) the repulsion loss to maximize the G-LMSM representativity; (2) the square root as an activation function to reprimand overconfidence; (3) fully-connected layer after a G-LMSM layer to represent more abstract objects and (4) temperature softmax to control model confidence and unbound the logit similarity values.
\end{itemize}

 \section{Related works}
  
%

In this section, we briefly review the recent literature of image set recognition and subspace-based learning.

\subsection{Image set recognition}

The task of image set recognition aims to recognize an object from an image set, classifying the set in one of a finite amount of defined categories. It has attracted extensive attention of research community~\cite{wang2018settoset,wang2017discriminative,lu2015multi,feng2016pairwise,wang2017discriminative}, and various applications have found success in this paradigm, such as face recognition~\cite{naseem2010linear, stallkamp2007video-based, zhou2002probabilistic} and object recognition~\cite{uijlings2013selective,li2009boosting,jarrett2009best}.

Zhao et al.~\cite{zhao2019review} categorized the methods that approach image set recognition into parametric and non-parametric models. Parametric models, e.g. manifold density divergence~\cite{arandjelovic2005face}, model sets and compare them by probability distributions and their divergences. Non-parametric models, however, do not assume a specific distribution and instead model the sets by a geometric object. The non-parametric models can be divided into: (1) linear subspace methods, such as the mutual subspace method (MSM)~\cite{maeda2004towards}, discriminant correlation analysis (DCC)~\cite{kim2009canonical}, constrained MSM (CMSM)~\cite{fukui2015difference}. (2) nonlinear manifold methods, such as the manifold-manifold distance (MMD)~\cite{wang2012manifold} and manifold discriminant analysis (MDA)~\cite{wang2009manifold}; and (3) affine subspace methods, e.g. affine hull based image set distance (AHISD)~\cite{cevikalp2010face} and sparse approximated nearest point (SANP)~\cite{hu2012face, hu2011sparse}.

\subsection{Subspace-based learning}

As seen in the previous section, linear subspaces can be employed as models for image sets. Although subspace-based learning intersects with the area of image set recognition, it also intersects a wide range of other fields, as many problems involve some kind of subspace structure, orthogonality or low-rank constraints, or subspace distances (e.g., canonical angles or projection norms). In most cases, the problems mathematical characteristics are expressed naturally using the Grassmann manifold. In the last few years, there have been growing interest in studying the Grassmann manifold to tackle new learning problems in computer vision, natural language processing, wireless communications, and statistical learning.
Some of the recent approaches of Grassmannian learning involve metric learning, Grassmann kernels, dictionary learning and capsules. We introduce some of these works below.

\textbf{Dictionary learning: }
Harandi et al.~\cite{harandi2013dictionary} proposes an approach based on sparse dictionary learning to learn linear subspaces, exploiting the Grassmann geometry to update the dictionary and handle non-linearity. Experiments on various classification tasks such as face recognition, action recognition, dynamic texture classification display advantages in discrimination accuracy compared to related methods.


Another method for face recognition based on dictionary learning and subspace learning (DLSL) is introduced in~\cite{liao_face_2019}. This new approach efficiently handles corrupted data, including noise or face variations (e.g., occlusion and significant pose variation). The DLSL uses a new subspace learning algorithm with sparse and low-rank constraints. Results obtained through experiments on FRGC, LFW, CVL, Yale B, and AR face databases reveal that DLSL achieves better performance than many state-of-the-art algorithms.



\textbf{Grassmannian kernels: }
Encouraged by the advantages of linear subspace representation, a discriminant learning framework has been proposed by Hamm et al.~\cite{hamm2008grassmann, hamm2008thesis}. In this method, Various Grassmann kernel functions are developed that can map a subspace to a vector in a kernel space isometric to the Grassmannian. The subspaces are handled as a point in this kernel space through the use of a kernel trick,  allowing feature extraction and classification. In the paper, kernel discriminant analysis is applied to the subspaces, a method called Grassmann discriminant analysis (GDA). In addition, experimental results on diverse datasets confirm that the proposed method provides competitive performance compared with state-of-the-art solutions.


In~\cite{harandi2011graph}, a Grassmannian kernel based on the canonical correlation between subspaces is proposed, improving the discrimination accuracy by estimating the local structure of the learning sets. More precisely, within-class and between-class similarity graphs are generated to describe intra-class compactness and inter-class separability. Experimental results obtained on PIE, BANCA, MoBo, and ETH-80 datasets confirm that the discriminant approach improves the accuracy compared to current methods.


Another kernel-based method~\cite{lincon2016egda} argues that Grassmann discriminant analysis (GDA) has a decrease in its discriminant ability when there are large overlaps between the subspaces. The enhanced GDA is proposed to resolve this issue, where class subspaces are projected onto a generalized difference subspace (GDS) before mapping them onto the manifold. In general terms, GDS removes the overlapping between the class subspaces, assisting the feature extraction and image set classification conducted by GDA. Hand shape and CMU face databases are employed to show the advantages of the proposed eGDA in terms of classification accuracy.



\textbf{Metric learning: }
Zhu et al.~\cite{zhu_towards_2018} employs nonlinear Riemannian manifolds for representing data as points, a practice widely adopted in the computer vision community. This work proposes a generalized and efficient Riemannian manifold metric learning (RMML), a general metric learning technique that can be applied to a large class of nonlinear manifolds. The RMML optimization process minimizes the geodesic distance of similar points and maximizes the geodesic distance of dissimilar ones on nonlinear manifolds. As a result, this procedure produces a closed-form solution and high efficiency. Experiments were performed using several computer vision tasks. Experimental results The experimental results show that RMML outperforms related methods in terms of accuracy.


Sogi et al.~\cite{sogi2020metric} propose a metric learning approach for image set recognition. The main task is to provide a reliable metric for subspace representation. The objective of the proposed metric learning approach is to learn a general scalar product space that provides improved canonical angles to estimate the similarity between subspaces. In addition, they present a formulation for dimensionality reduction by imposing a low-rank constraint. Experimental results are provided using video-based face recognition, multi-view object recognition, and action recognition datasets.


Luo et al.~\cite{luo_robust_2019} addresses the problem that current metric learning solutions employ the L2-norm to estimate the similarity between data points. The authors argue that although practical, this approach may fail when handling noisy data. Then, they introduce a robust formulation of metric learning to solve this problem, where a low-dimensional space is utilized for producing a suitable error estimation. In addition to a robust framework, the authors present generalization guarantees by developing the generalization bounds employing the U-statistics. The proposed framework is evaluated using six benchmark datasets, obtaining high discrimination accuracy compared to related methods.

\textbf{Capsules: }
Capsule projection (CapPro), proposed recently by~\cite{zhang2018cappronet}, is a neural network layer constructed by subspace capsules. In CapPro subspaces have been incorporated in neural networks, concretely as parameters of capsules. The general idea of capsule was first proposed by~\cite{hinton2011transforming, sabour2017dynamic}, and since then, much effort has been made to seek more effective capsule structures as the base for the next generation of deep network architectures. Capsules are groups of neurons that can learn an object's variations, such as deformations and poses effectively. CapPro obtains an output vector by performing an orthogonal projection of feature vectors $\bm x$ onto the capsule subspace. The intuition behind CapPro is that it learns a set of such capsule subspaces, each representing a class. Note that CapPro cannot handle a set in a natural manner, and it does not perform subspace matching.
Our proposed models can be seen as extensions of CapPro, in the sense that their forward map correspond to CapPro in the case the input consists of a single vector $\bm x \in \mathbb{R}^{d}$.

\section{A review of learning subspaces} \label{sec:review}

In this section, we review the theory and algorithms of classic learning subspace methods.

\subsection{Problem formulation}

The learning subspace methods are Bayesian classifiers. The task they solve is defined as follows: Let $\{ \bm x_i , y_i \}_{i=1}^{N}$ be a set of $N$ reference image feature vectors with $d$ variables $\bm x_i$, paired with respective labels $y_i~\in~\{1,\cdots,C\}$. We want to learn a set of class subspaces, each represented by orthogonal basis matrices $\{ \bm V_c \}_{c=1}^{C}$, that correctly predicts the class of a novel vector $\bm x \in \mathbb{R}^d$.

\subsection{Subspace method}

The subspace method (SM)~\cite{watanabe1967evaluation, iijima1972theoretical} is a generative method that represents classes by subspaces. Let $\{ \bm V_c \}_{c=1}^{C}$ be the bases of $m$-dimensional class subspaces. In SM, the bases are obtained from principal component analysis (PCA) without centering of each class reference data, i.e., each class subspace is computed independently by an eigenvalue decomposition of the class auto-correlation matrix $\bm A$:
\begin{gather}
    \sum_{y_i \in c} \bm x_i \bm x_i\tp = \bm A =  \bm U \bm \Lambda \bm U \tp, \\
    \bm V_c = \bm U_{1:m}.
\end{gather}
Since $\bm A$ is symmetric positive semidefinite, it has eigenvectors $\bm{u}_{1}, \ldots, \bm{u}_{r}$, and corresponding real eigenvalues $\lambda_{1}, \ldots, \lambda_{r}$, where $r = \operatorname{rank} \bm A$ and $m \leq r$. 
Without loss of generality we can assume $\lambda_{1} \geq \lambda_{2} \cdots \geq \lambda_{r}$. 
Here,  $\bm U = [\bm{u}_{1}, \bm{u}_{2}, \cdots, \bm{u}_{r}]$ is the matrix of eigenvectors and $\bm \Lambda \in \mathbb{R}^{r \times r}$ is a diagonal matrix where each $k$-th diagonal entry is an eigenvalue $\lambda_k$. $\bm U_{1:m}$ represents the first $m$ eigenvectors of $\bm A$. The subspace dimension $m$ is selected as a hyperparameter. 

\subsection{Learning subspace method}

The learning subspace method (LSM), proposed by Kohonen~\cite{kohonen1979spectral} introduces concepts of learning to SM by learning the class subspaces with an iterative algorithm. To initialize, the subspaces are computed as in SM. Then, the reference samples are classified iteratively. For a training sample $\bm x$ of class $y$, let the LSM classifier prediction be written as $q~\in~\{1,\cdots,C\}$. The prediction $q$ is obtained as:
\begin{equation}
    q = \text{argmax}_{c} \| \bm V_c^\top \bm x\|.
    \label{eq:lsm_classification}
\end{equation}
When a training sample is correctly classified, i.e., $q = y$, the subspace representation is \emph{reinforced} by rotating the class subspace $\bm V_y$ slightly towards the sample vector. Let $t$ denote the current iteration and $\alpha$ a learning rate, then:
\begin{equation}
     \bm V_y(t+1) = (\bm I + \alpha \bm x \bm x^\top ) \bm V_y(t) .
\end{equation}
When a sample is misclassified as being of some other class, i.e., $q \neq y$, the misclassified class subspace $\bm V_q$ is \emph{punished} by rotating it away from the sample vector, while the class subspace $\bm V_y$ is rotated towards the sample vector so that it is correctly classified. Let $\beta, \gamma$ be learning rates, then:
\begin{gather} 
     \bm V_y(t+1) = (\bm I + \beta \bm x \bm x^\top )\bm V_y(t) ,  \\
    \bm V_q(t+1) = (\bm I - \gamma \bm x \bm x^\top )\bm V_q(t)  .
\end{gather}


\subsection{Average learning subspace method}

This subspace updating procedure is simple, yet theoretically sound and clearly superior to a priori methods such as SM and the Bayes classifier with normality assumptions up to a certain size~\cite{oja1983subspace}. Yet, LSM is sensitive to the order of samples in iterations; the representation induced by the first samples tends to be canceled by subsequent samples, making the subspace wander around instead of steadily move towards local optima. To obtain smoother behavior and faster convergence, the averaged learning subspace method (ALSM) has been proposed, which considers the update based on the expectation of samples rather than each sample individually.

The ALSM update, also referred in this paper as Oja's update, assumes that the sampling order of the training data is statistically independent. For a class subspace $c$, let the three potential statistical outcomes be denoted as follows. $H_{\text{CC}}$ stands for the event of a \textit{correct classification} into class $c$ i.e. $c = q = y$. Then, there are two cases of misclassification ($q \neq y$): $H_{\text{FN}}$ denotes a false negative ($c = y$), and $H_{\text{FP}}$ denotes a false positive ($c = q$).
Then, the ALSM update is given for a class subspace $c$ as follows:
\begin{equation}
\begin{aligned}
   \bm V_c(t+1) & = (\bm I + \alpha E[\bm x_i \bm x_i^{\top} | H_{\text{CC}} ] p(H_{\text{CC}} ) \\
    & + \beta \sum_{c = y_i \neq q} E[\bm x_i \bm x_i^{\top} | H_{\text{FN}}] p(H_{\text{FN}}) \\
    & - \gamma \sum_{c = q\neq y_i} E[\bm x_i \bm x_i^{\top} | H_{\text{FP}}] p(H_{\text{FP}})
    )\bm V_c(t).
    \label{eq:ojaupdate}
\end{aligned}
\end{equation}

That is, the expected updated subspace $\bm V_c(t+1)$ given the previous subspace $\bm V_c(t)$ is a rotation of $\bm V_c(t)$. This rotation is a weighted average of three cases: (1) the rotation to \emph{reinforce} the correctly classified cases ($c=q$); (2) to \emph{punish} false negatives ($c \neq q$ a sample of class $c$ misclassified as $q$); (3) to \emph{punish} false positives ($q \neq c$ a sample of class $q$ misclassified as $c$). Note that $\bm V_c(t+1)$ will not necessarily be an orthogonal basis matrix, requiring Gram-Schmidt orthogonalization. We omit this from notation, being implicitly necessary for all Oja's updates.

\section{Proposed methods}

In this section, we describe the algorithm of the proposed LMSM and G-LMSM.

\subsection{Basic idea}

The ALSM has two major shortcomings that makes it unfit for image set recognition: (1) their input can only be a single vector, rather than a set of vectors, and (2) they update the reference subspaces with the Euclidean gradient.
Therefore, in the following we propose two methods to address these problems by: (1) matching subspace to subspace and (2) performing updates with the Grassmannian gradient.

In this section, we generalize Oja's learning subspace framework to an image set setting, where now one sample represents one image set. Concretely, we replace a sample vector $\bm x \in \mathbb{R}^{d}$ by a subspace spanning an image set $\bm X \in \mathbb{R}^{d \times m}$ with $m$ basis vectors, in the same manner as realized by MSM. We call this algorithm the learning mutual subspace method (LMSM).
The "mutual" word refers to the introduction of a matching between two subspaces based on their canonical angles. Such a method would be a natural extension of the learning subspace methods, but to the best of our knowledge, has not been proposed, most likely because finding appropriate ALSM's parameters $\alpha, \beta$, and $\gamma$ would become increasingly exhaustive under a heuristic subspace matching paradigm. Instead, we rewrite Oja's rule as a gradient update so that LMSM can exploit the backpropagation algorithm to enjoy learning stability similar to neural networks.

The main differences between ALSM and LMSM are:
(1) LMSM performs subspace matching, whereas ALSM performs vector matching;
And (2) ALSM learns its reference subspaces by performing heuristic parameter updates, while the LMSM performs explicit gradient updates.


Then we propose the Grassmann learning subspace method (G-LMSM), further generalizing ALSM and LMSM. Unlike them, G-LMSM updates the gradient through the Riemannian SGD, which smoothly enforces the subspace constraint by correcting the Euclidean gradient to a Grassmannian gradient, keeping the reference subspace bases orthogonal. In contrast, the previous methods recompute the bases based on the Gram-Schmidt orthogonalization, which can lead to unstable convergence.

\subsection{Oja's rule reformulation}

Before describing the proposed methods, we reformulate Oja's rule as a gradient update.
Now we turn our attention to Oja's update (equation~\ref{eq:ojaupdate}). In our formulation, we assume that all learning rates are equal; as a result, equation~\ref{eq:ojaupdate} can be condensed as follows:
\begin{align}
    \bm V_c(t+1) &= (\bm I + \alpha \sum_{i=1}^{N} \iota(c,q_i,y_i) \bm x \bm x\tp)\bm V_c(t) \\
    &= \bm V_c(t) + \alpha \sum_{i=1}^{N} \iota(c,q_i,y_i) \bm x \bm x^{\top} \bm V_c(t),
\end{align}
where $\iota(c,q,y)$ is the \emph{indicator} function for ALSM: for false negatives and correct classification it outputs $+1$ and for false positives it indicates $-1$, otherwise it indicates $0$.

Without loss of generality, assume a sample of size $1$:
\begin{align}
    \bm V_c(t+1) = \bm V_c(t) + \alpha \iota(c,q,y) \bm x \bm x\tp \bm V_c(t).
    \label{eq:mlsmupdate}
\end{align}
Let the indicator function of the Oja's update be the gradient of a \emph{loss} function $L$, i.e., $\frac{\dd L}{\dd s_c} = \iota(c,q,y)$. From the definition of $\iota$, it is possible to determine that $L$ corresponds to the common cross-entropy loss function plus an extra term. Throughout this paper, we use $L$ to be the cross-entropy.

We prove that the remaining term $\bm x \bm x\tp \bm V_c(t)$ is the derivative of the squared vector projection onto the subspace $ \bm V_c$ (up to scaling), the similarity used by ALSM to match subspace and vector on the learning subspaces classification (eq.~\ref{eq:lsm_classification}).
\begin{proposition}
	$\frac{\dd s}{\dd \bm V} = 2\bm x \bm x\tp \bm V$ for $s = \| \bm V\tp \bm x  \|_2^2$.
\end{proposition}
\begin{proof}
	\begin{align}
	\frac{\dd s}{\dd \bm V} &= \frac{\dd }{\dd \bm V} \| \bm V\tp \bm x  \|_2^2 \\
	&= 2 \| \bm V\tp \bm x  \|_2 \frac{\dd }{\dd \bm V} \| \bm V\tp \bm x  \|_2 \\
    &= 2 \| \bm V\tp \bm x  \|_2 \| \bm V\tp \bm x  \|_2^{-1} \bm x \bm x\tp \bm V \\
    &=  2 \bm x \bm x\tp \bm V.
\end{align}
\end{proof}
By using this result and inverting the signal of $\alpha$, we rewrite Oja's rule as an usual gradient update:
\begin{align}
	\bm V_c(t+1) &= \bm V_c(t) - \alpha \frac{\dd L}{\dd s_c} \frac{\dd s_c}{\dd \bm V_c}.
	    \label{eq:CapProisALSM}
\end{align}
This gradient form of the Oja's rule is the basis to develop LMSM and G-LMSM learning methods. The reasoning behind this reformulation is that many modern learning algorithms used for pattern recognition, especially neural networks, are trained with gradient descent methods. By using a gradient update, the methods we will develop in the next subsections can be flexibly employed within modern neural network frameworks in a plugin manner, without the need for reconfiguring the learning strategy. 

\subsection{Problem formulation}

Our learning problem is defined as follows: let $\calli H = \{ \calli I_l \}_{l=1}^{n}$ denote a set of $n$ images $\calli I_l$ of size $w \times h \times c$, where $w$ is width, $h$ is height and $c$ denotes the number of channels. Given training sets $\{ \calli H_i , y_i \}_{i=1}^{N}$, paired with respective labels $y_i$, we want to learn a model that correctly predicts the class of a novel set $\calli H$. Note that the number of images $n$ in a set does not need be the same among all sets. 
In this section, we consider the processing of a single instance $\calli H$ to keep the notation simple. 


In our framework, we model a set  $\calli H$ by means of a subspace $\bm \chi \in \mathbb{G}(d,m)$, where $\mathbb{G}(d,m)$ denotes the Grassmann manifold (or Grassmannian) of $m$-dimensional linear subspaces in $\mathbb{R}^{d}$. The Grassmannian is an $m(d-m)$-dimensional manifold, where a point corresponds to a $m$-dimensional subspace. $\bm \chi$ is spanned by the orthogonal basis matrix  $\bm{X} \in \mathbb{S}(d,m)$, where $\mathbb{S}(d,m)$ denotes the set of orthonormal matrices of shape $\mathbb{R}^{d \times m}$, called the compact Stiefel manifold. To compute $\bm{X}$, we utilize noncentered PCA on the vectorized images or their CNN features, more being explained on Section~\ref{sec:netarch}, where all framework parts are put together.

\subsection{Proposed learning methods}

In this subsection, we explain the matching framework of the LMSM and G-LMSM.
Our subspace matching is defined as the map $s_{\bm{\upsilon}}: \mathbb{G}(d,m) \rightarrow \mathbb{R}$ from the Grassmann manifold to the reals. $s_{\bm{\upsilon}}$ is parameterized by the \textit{reference subspace} $\bm{\upsilon} \in \mathbb{G}(d,p)$ spanned by a matrix $ \bm{V} \in \calli{S}(d,p)$. Given an \textit{input subspace} $\bm{\chi} \in \mathbb{G}(d,m)$,  spanned by an orthogonal basis matrix  $\bm{X} \in \mathbb{S}(d,m)$, $s_{\bm{\upsilon}}(\bm{\chi})$ is defined as the sum of the squared cosines of the canonical angles $\theta_i$ between $\bm{\chi}$ and $\bm{\upsilon}$:
\begin{equation}
    s_{\bm{\upsilon}} = \sum_{i=1}^{r} \cos^2 \theta_i(\bm{\chi}, \bm{\upsilon}),
    \label{eq:G-LMSMdef}
\end{equation}
where $r = \min(p,m)$. The canonical angles $\{0\leq \theta_1,\cdots,\theta_{r}\leq\frac{\pi}{2}\}$ between $\bm{\chi}$ and $\bm{\upsilon}$ are recursively defined as follows~\cite{hotelling1992relations, afriat1957orthogonal}.
\begin{eqnarray}
\cos{\theta_i}=\max_{{\bf u}\in {\bm{\chi}}}\max_{{\bf w}\in {\bm{\upsilon}}}{\bf u}^\mathrm{T}{\bf w}={\bf u}_i^\mathrm{T}{\bf w}_i, \\
s.t.\ \|{\bf u}_i\|_2=\|{\bf w}_i\|_2=1,{\bf u}_i^\mathrm{T}{\bf u}_j={\bf w}_i^\mathrm{T}{\bf w}_j=0,i\neq j, \nonumber
\end{eqnarray}
where ${\bf u}_i$ and ${\bf w}_i$ are the canonical vectors forming the $i$-th smallest canonical angle $\theta_i$ between $\bm{\chi}$ and $\bm{\upsilon}$. The $j$-th canonical angle $\theta_j$ is the smallest angle in the direction orthogonal to the canonical angles $\{\theta_k\}_{k=1}^{j-1}$.

This optimization problem can be solved from the orthogonal basis matrices of subspaces $\bm{\chi}$ and $\bm{\upsilon}$, i.e.,
$\cos^2 \theta_i$ can be obtained as the $i$-th largest singular value of $\bm{X}\tp\bm{V}$~\cite{edelman1998geometry}.
From this approach, our output $s_{\bm{\upsilon}}$ can be obtained as the sum of the squared singular values of $\bm{X}\tp\bm{V}$, written as $\sum_{i=1}^{r} \delta_i^2(\bm{X}\tp\bm{V})$.
We utilize the spectral definition of the Frobenius norm to simplify our operation, redefining the subspace matching as:
\begin{equation}
    s = \sum_{i=1}^{r} \delta_i^2(\bm{X}\tp\bm{V}) = \| \bm{X}\tp\bm{V}\|_F^2 = \tr \bm{V}\tp \bm{X} \bm{X}\tp \bm{V}.
    \label{eq:frobdef}
\end{equation}
As such, in our framework, we compute the matching with this matrix trace. Note that we omit the subscript from $s_{\bm{\upsilon}}$, as it is clear it is parameterized by $\bm{\upsilon}$. We follow this to simplify notation.


This definition of the LMSM similarity is a natural extension of ALSM similarity  (eq.~\ref{eq:lsm_classification}) that works for subspaces. In the special case of set recognition in which the input is a single vector $\bm x$ instead of a subspace, we obtain $s = \| \bm{x}\tp\bm{V}\|_2^2$.



LMSM has a set of learnable parameters composed of $K$ reference subspaces as matrices $\{ \bm V_j\ \in \calli{S}(d,p) \}_{j=1}^K$.
It is a function $\mathbb{G}(d,m) \rightarrow \mathbb{R}^K$ from the Grassmann manifold to a latent space of dimension $K$. 

\textbf{Matching:} Given an input subspace basis $\bm X$, its matching to the reference subspaces is defined as:
\begin{equation}
    \bm{s} = 
       \begin{bmatrix}
       \tr \bm{V}_1\tp \bm{X} \bm{X}\tp \bm{V}_1 \\
       \vdots \\ 
       \tr \bm{V}_K\tp \bm{X} \bm{X}\tp \bm{V}_K \\
       \end{bmatrix}
       \in \mathbb{R}^K,
    \label{eq:G-LMSMlayer}
\end{equation}
which we write as $s_j = \tr \bm{V}_j\tp \bm{X} \bm{X}\tp \bm{V}_j$ for compactness. Each reference subspace represents a class, so that each value $s_j$ represents the match of the input subspace to the patterns in a reference subspace.

\textbf{Gradient:} The gradient of parameters and input can be obtained by differentiating~(\ref{eq:G-LMSMlayer}) and applying the chain rule, given the gradient $\dot s_j $ with respect to a loss $L$. We obtain the following parameter update:
\begin{equation}
\dot{\bm{V}_j} = (2\dot{s}_j \bm{V}_j\tp \bm{X}\bm{X}\tp )\tp= 2\dot{s}_j \bm{X}\bm{X}\tp \bm{V}_j.
\label{eq:vderiv}
\end{equation}
The input gradient is then:
\begin{equation}
\dot{\bm{X}} = 2\dot{s}_j \bm{V}_j\bm{V}_j \bm{X}\tp.
\end{equation}

\subsection{Learning mutual subspace method}

Both LMSM and G-LMSM perform matching using the similarity presented above. They differ in their learning update, on how they consider the geometrical structure. LMSM uses Euclidean updates whereas G-LMSM uses updates on the Grassmannian.
  
LMSM employs conventional stochastic gradient descent (SGD). For that, we set the reference subspaces $\bm{V}_j$ to be unconstrained parameters (not orthogonal bases), which can be updated by SGD. Then, when matching subspaces (eq.~\ref{eq:G-LMSMlayer}) we rewrite our similarity to an equivalent form as follows:
\begin{equation}
    s = \tr \bm{X}\tp \bm{V} (\bm{V}\tp \bm{V} + \epsilon \bm{I})^{-1} \bm{V}\tp \bm{X}.
    \label{eq:frobdef}
\end{equation}
$\epsilon \bm I$ is a very small-valued multiplier of the identity matrix to regularize possible computational instabilities.
We utilize the pseudoinverse of $\bm{V}_j$ to ensure that the similarity reflects the defined function of canonical angles, and is not affected by other quantities embedded in the correlation matrices of nonorthogonal bases, such as norms of vectors and their correlation.
Note that the input subspace basis $\bm{X}$ is normally is computed by PCA as a orthogonal basis.


\subsection{Grassmann learning subspace method}

G-LMSM performs its learning on the Grassmannian instead of the Euclidean space, updating the reference subspaces by the Riemannian stochastic gradient descent (RSGD)~\cite{bonnabel2013stochastic, becigneul2018riemannian}. This manifold aware update allows the method to keep the reference subspace bases consistently orthogonal while maintaining a stable and efficient learning, without the need of tricks such as the pseudoinverse. From a geometry perspective, the method enforces the iterates $\bm{V}_j(t)$ to always be a member of the Grassmann manifold, i.e., avoiding the parameter search line to leave the manifold. The RSGD update consists of two steps: (1) transforming the Euclidean gradient $\dot{\bm{V}}_j$ to a Riemannian gradient $\pi(\dot{ \bm V}_j)$ tangent to the Grassmannian, that is, the closest vector to $\dot{\bm{V}_j}$ that is also tangent to $\calli G$  at $\bm{V}_j$. That can be obtained by the orthogonal projection of $\dot{\bm{V}_j}$ to the horizontal space at $\bm{V}_j$ as:
\begin{equation}
    \pi(\dot{ \bm V}_j) = (\bm I  - \bm{V}_j \bm{V}_j\tp)\dot{\bm{V}_j}.
\end{equation}
The next step is (2) updating the reference subspace by the Grassmannian exponential map:

\begin{equation}
\label{eq:exp-update}
\bm{V}' = \Exp_{\bm{V}}   \lambda\dot{\bm{V}} = \orth(\bm{V}\bm{Q}(\cos \bm{\Theta} \lambda)\bm{Q}^\top + \bm{J}(\sin \bm{\Theta} \lambda)\bm{Q}^\top)
.
\end{equation}
Here, $\bm{J}\bm{\Theta} \bm{Q}^{\top} = \pi(\dot{ \bm V})$ is the compact singular value decomposition (SVD) of the gradient, and $\lambda$ is a learning rate.
This update can be regarded as a geodesic walk towards the opposite direction to the Riemannian gradient (direction of descent), landing on a more optimal point $\bm{V}'$.


\subsection{Loss functions}

In their basic form, both LMSM and G-LMSM are trained using the cross-entropy (CE) as an objective:
First, the similarities are mapped by softmax (similarities are used as logits):
\begin{equation}
	\bm q_c  =  \frac{\exp(s_c)}{\sum_{j=1}^{C} \exp( s_j )}.
	\label{eq:softmax}
\end{equation}

Then, the CE loss can be defined as follows.
Let $\Omega:~\{ 1, \ldots, C \}~\rightarrow~\{ 0,1 \}^C$ be a function such that $\| \Omega(p) \|_1 = 1$. This function maps an integer $p$ to a \textit{one-hot} vector, i.e., a vector of zeroes where only the entry $p$-th entry is set to one. With that function we can define two vectors: the one-hot prediction vector $\bm q = \Omega(q)$ and the label vector $\bm y = \Omega(y)$. Then the CE loss follows:
\begin{equation}
	L_\text{CE}(\{ \bm V\}, \bm x, y) = - \frac{1}{\tau} \sum_{c=1}^{C} \bm y_c \log (\bm q_c).
	\label{eq:celoss}
\end{equation}


However, the cross-entropy does not handle one possibility in the learning subspace methods. Their capacity for representation and discrimination can be impaired if the reference subspaces are highly overlapped, as most multiple activation values will tend to full activation $r$. This is more likely to happen when the class spread is very large, or the dimension $r$ is large, and $d$ is small. To tackle this possibility, we propose a repulsion loss to use in G-LMSM architectures. Let $L_{CE}(\calli H)$ be the cross-entropy loss given the set $\calli H$. The total loss is then given by $L(\calli H) = L_{CE}(\calli H) + \gamma L_{RP}(\{ \bm V_j\})$, where $\gamma$ is a control parameter and:
\begin{equation}
    L_{RP} = \frac{1}{K^2} \sum_{i=1}^K \sum_{j=1}^K (\frac{1}{r} \| \bm V_i\tp \bm V_j \|_F^2 - \delta_{ij})^2.
\end{equation}
Here, $\delta_{ij}$ is the Dirac delta. The idea of the repulsion loss is to guide the learning to mutually repel the reference subspaces, so that the class representations are unique and well separated.

\subsection{Nonlinear activation}

A possible scenario when training a G-LMSM is that, as the value $s_j$ approaches zero, the reference subspaces may offer little contribution to the model. This situation can impair some subspaces from ever contributing, or in the softmax, can lead to sparse, confident predictions that result in misclassification of borderline samples. To avoid this possibility, we can extend the G-LMSM with an elementwise nonlinear activation $\phi$ as follows:
\begin{equation}
    s_j = \phi(\tr \bm{V}_j\tp \bm{X} \bm{X}\tp \bm{V}_j).
\end{equation}

There are numerous choices of $\phi$ for G-LMSMs. In this paper, we choose to use the square root, that is, $s = \| \bm{X}\tp\bm{V}\|_F$. The main reason is that it is a more natural extension of ALSM. Another reason is that it is straightforward nonlinearity that has been used to regularize neural networks~\cite{yang2018square}. Concretely, the square root will decrease the value of very high activations with high intensity, while not decreasing as much lower activations.


\subsection{Softmax with temperature}

Another approach to handling overconfidence is to use softmax with \textit{temperature} $T$, used by~\cite{guo2017calibration, hinton2015distilling,liang2017enhancing}. In our framework, we utilize the \textit{inverse-temperature} hyperparameter $\tau = 1/T$ to scale the logits before applying softmax. This learned scaling has the potential to free the bounded logits ($0 \leq s_j \leq \sqrt{r}$), allowing them to achieve proper classification margin.
We can use it with our G-LMSM simply as $\exp( \tau s_j) / \sum_{j'=1}^{C}\exp( \tau s_{j'})$ . Agarwala et al.~\cite{agarwala2020temperature} offers a detailed analysis of the underpinning learning mechanisms of softmax with temperature.

The intuition of this idea is that by scaling the values of the activation we can control the level of confidence of our model when it makes predictions. When the temperature is $\tau = 1$, we compute the unscaled softmax, leading to the basic G-LMSM with softmax. When the temperature is greater than $1$ the input of softmax is a larger value. Performing softmax on larger values makes the model more confident, i.e., a smaller activation is needed to yield a high posterior probability. However, it also makes the model more conservative when looking at data, that is, it is less likely to sample from unlikely candidates. Using a temperature less than $1$ produces a less confident probability distribution over the classes, where the probabilities are more distributed over categories. The model is less conservative on data, resulting in more diverse representation. However, if overdone the model is also more prone to mistakes.

Since the temperature depends largely on the task, in our framework, we allow the temperature to be learned automatically as a network parameter. 
\section{Experiments}
\begin{figure*}[htb]
        \begin{subfigure}{.5\textwidth}
          \centering
          \includegraphics[width= 8 cm]{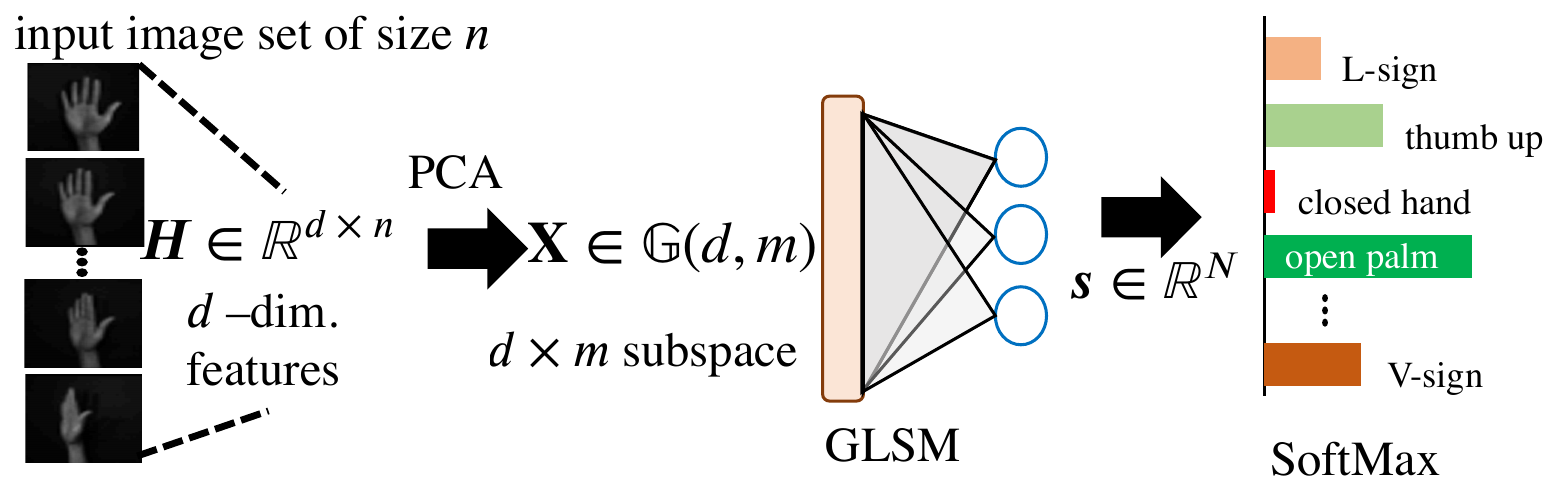}
          \caption{G-LMSM+softmax.}
          \label{fig:mcsx}
        \end{subfigure} \quad
        \begin{subfigure}{.5\textwidth}
          \centering
          \includegraphics[width= 8 cm]{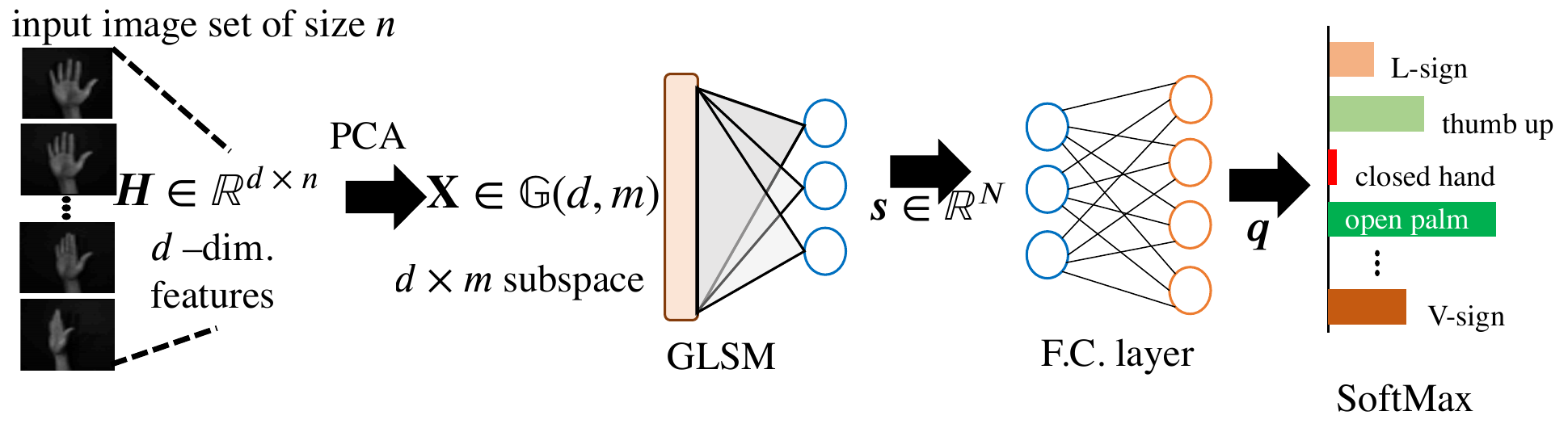}
          \caption{G-LMSM+FC.}
          \label{fig:mcfc}
        \end{subfigure}
        \begin{subfigure}{.5\textwidth}
          \centering
          \includegraphics[width= 8 cm]{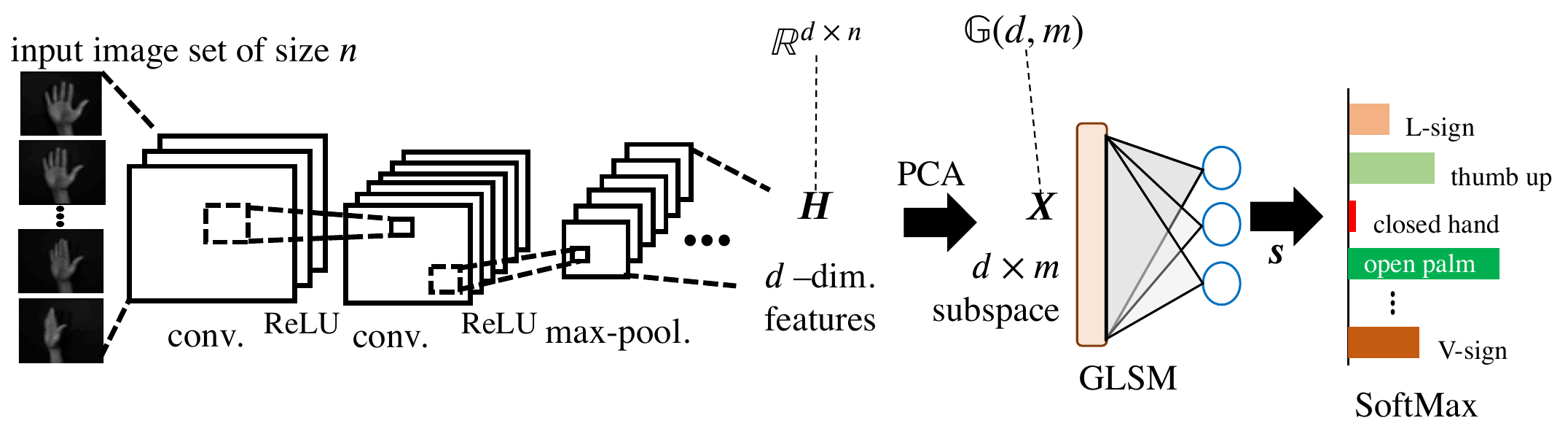}
          \caption{conv+G-LMSM+softmax.}
          \label{fig:convmcsx}
        \end{subfigure}
        \begin{subfigure}{.5\textwidth}
          \centering
          \includegraphics[width= 8 cm]{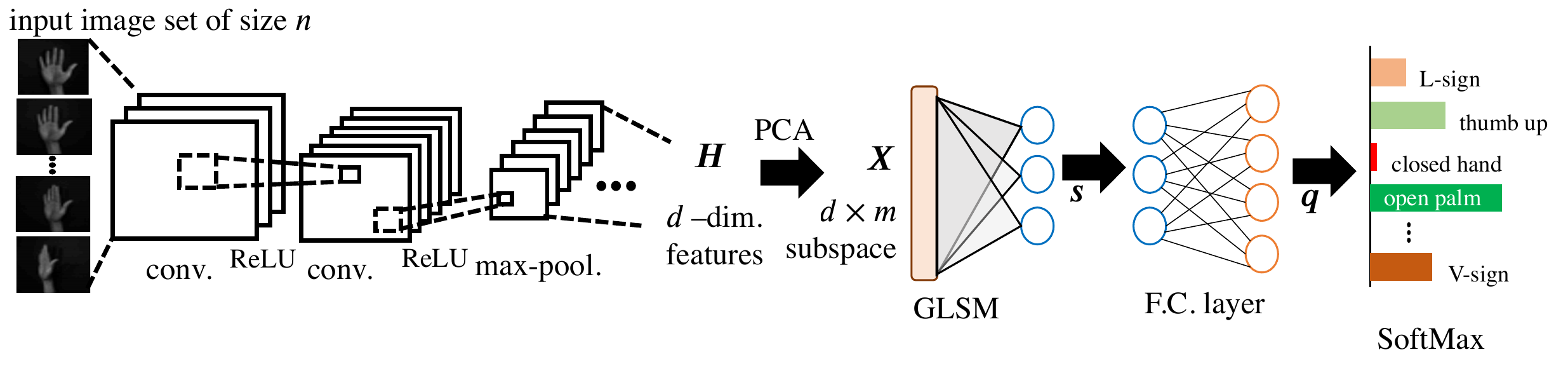}
          \caption{conv+G-LMSM+FC.}
          \label{fig:convmcfc}
        \end{subfigure}
\caption{Conceptual diagrams of the proposed G-LMSM architectures.}
\label{fig:archs}
\end{figure*}

\subsection{Network Architectures}
\label{sec:netarch}

In the following experiments, we use network architectures containing the proposed methods as a layer to approach image set recognition. The architectures  are different in two separate design choices: (1) choice of set features, for example, raw images or CNN features; and (2) function of G-LMSM in the pipeline as a classifier or as a feature extractor. From these options, we can obtain the following architectures, shown on Figure~\ref{fig:archs}.

\textbf{G-LMSM+softmax}: This is the simplest architecture, consisting of a G-LMSM layer and a softmax applied to the activation $\bm s$. In this setting, the input set is preprocessed into a subspace as follows: Let the matrix $\bm H \in \mathbb{R}^{d \times n}$ contain the set images as its columns, where $d = w*h*c$. Then, we compute a subspace from the set using noncentered PCA, i.e., $\bm H \bm H\tp = \bm U \bm \Sigma \bm U^T$. The subspace basis $\bm X \in \mathbb{R}^{d \times m}$ consists of the $m$ leftmost columns of $\bm U$, that is, the columns with the highest corresponding eigenvalues.
The subspace basis $\bm X$ is input in the G-LMSM layer, obtaining the activation $s \in \mathbb{R}^C$. Each reference subspace $\bm V_j$ represents one class, and the softmax of $s_j$ corresponds to the probability of that class given the samples. This model corresponds to CapPro in the case the input consists of a single vector $\bm x \in \mathbb{R}^{d}$.

\textbf{G-LMSM+FC}: This variation extends G-LMSM to act as a feature extractor rather than simply as a classifier, by processing $s$ further through a fully-connected layer before applying softmax. The number of reference subspaces $K$ becomes a free parameter, and the FC weights are $\bm W \in \mathbb{R}^{K \times C}$ and $b \in \mathbb{R}^C$.

\textbf{CNN backbone G-LMSM+softmax}: Given the set $\calli H$ as input, we process each image $\calli{I}_l$ through a convolutional neural network backbone $f$, obtaining features $\bm h_l = f(\calli{I}_l)$. Then, we can create $\bm H \in \mathbb{R}^{d \times n}$ with normalized $\bm h_l$ as columns. From here we either (1) perform noncentered PCA to obtain $\bm X$ or (2) approximate the subspace projection matrix by the autocorrelation matrix (AC) $\bm H \bm H\tp$. PCA is exact, but requires an SVD decomposition, which although can be processed in an end-to-end manner by utilizing the SVD gradient updates proposed by~[Townsend], can be unstable and lead to exploding gradients. The AC is not exact, as it does not yield a subspace projection matrix, but it is a reasonable approximation here since it quite fast and stable in end-to-end processing.

In this paper, we utilize two backbones: VGG4 and ResNet18 without the last stage of FC layers. These backbones act as feature extractors, while PCA combines the features into a subspace and a G-LMSM in the last layer acts as a classifier.

\textbf{CNN backbone G-LMSM+FC}: This corresponds to extracting CNN features from the set images, and then apply G-LMSM+FC, where G-LMSM here functions as a feature extractor.

\subsection{Ablation study on Hand Shape Recognition}

We conducted an ablation study experiment of Hand shape recognition with the Tsukuba hand shape dataset. The objective is twofold: (1) show that the proposed methods can classify objects by achieving a meaningful result, and (2) understand the effects of subspace dimensions and G-LMSM's function as a classifier and as a feature extractor.

\begin{figure}[t]
\centering
\includegraphics[width=\linewidth]{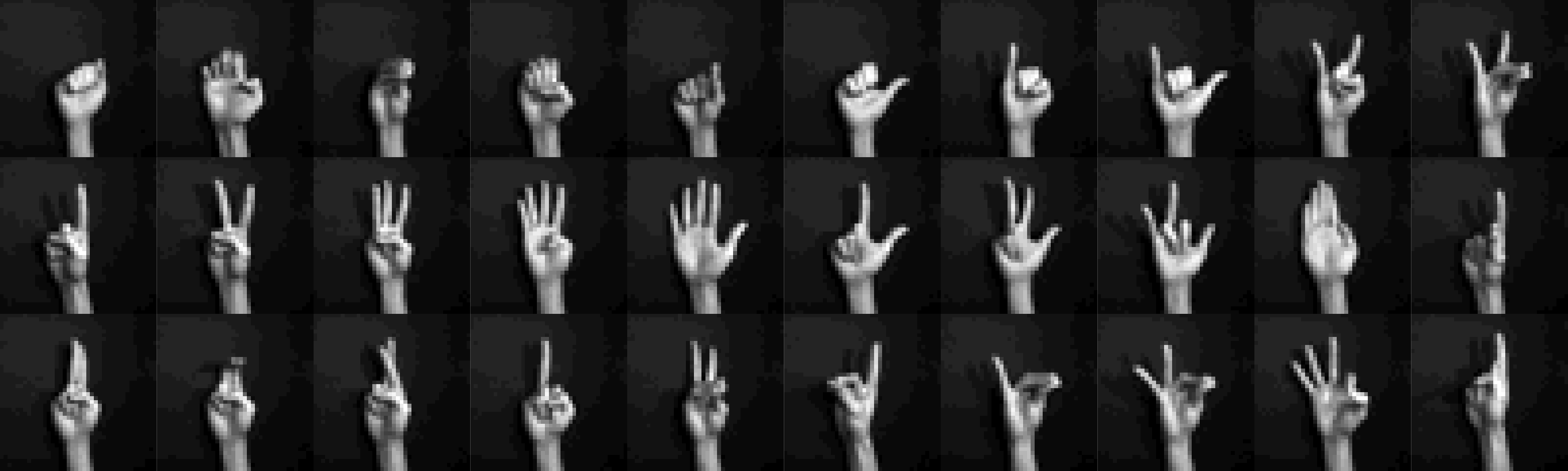}
\caption{Sample images from the Tsukuba hand shape dataset, displaying the $30$ classes of hand shapes.}
\label{fig:samplehand}
\end{figure}

\textbf{Dataset}:
This dataset contains $30$ hand classes $\times$ 100 subjects, each of which  contains 210 hand shape images, consisting of $30$ frames $\times$ 7 different viewpoints. Example images can be seen in Fig.~\ref{fig:samplehand}. For each subject, we randomly created 6 sets with 5 frames from each viewpoint, so that each subject has 6 image sets of 35 images. In summary, there are a total of 18000 image sets in this dataset, each set containing image information from 7 camera viewpoints. In the experiments, all the images were resized to $24 \times 24$ pixels.

\textbf{Settings}: We evaluated the performance of the proposed G-LMSM in the classification problem of $30$ types of hand shapes. We used the image sets of $70$ subjects as training sets, holding a subset of $15$ subjects for validation. The remaining $15$ subjects were used as testing sets.

\subsubsection{Effect of Optimization algorithm}
%

\begin{table*}[]
  \centering
  \caption{Results of the experiment on the Tsukuba hand shape dataset.}
    \begin{tabular}{l l l r}
    \toprule
    Feature &       & Method & \multicolumn{1}{l}{Accuracy (\%)} \\
    \midrule
    \multirow{5}[4]{*}{No backbone} & Baselines & MSM   & 67.10 \\
\cmidrule{2-4}          & \multirow{4}[2]{*}{Proposed} & G-LMSM+softmax & 68.87 \\
          &       & G-LMSM+FC & 89.98 \\
          &       & G-LMSM+FC (repulsion) & 89.52 \\
          &       & G-LMSM+sqrt+FC & 89.69 \\
    \midrule
    \multirow{8}[4]{*}{With backbone (VGG4)} & \multirow{4}[2]{*}{Baselines} & Set average & 97.54 \\
          &       & Bilinear & 83.81 \\
          &       & Softmax average & 97.48 \\
          &       & CapPro average & 96.80 \\
          &       & Nonorth. matching & 96.80 \\
\cmidrule{2-4}          & \multirow{4}[2]{*}{Proposed} & LMSM+softmax & 98.72 \\
          &       & PCA+G-LMSM+softmax & 99.39 \\
          &       & PCA+G-LMSM+FC & 99.06 \\
    \bottomrule
    \end{tabular}%
  \label{tab:exp_hand}
\end{table*}%

First we compare LMSM and G-LMSM to evaluate the choice of optimization algorithm, conventional SGD or Riemannian SGD.
As a baseline for comparison, we also evaluate a variation of LMSM where for the input set representation $\bm{X}$ we use the normalized features themselves as a basis, i.e., $\bm H = \bm{X}$, yielding the method dubbed "nonorth. matching". Note that in this case, the similarity might not correspond to the canonical angles between $\sspan(\bm H)$ and $\sspan(\bm V)$.


Table~\ref{tab:exp_hand} shows the results. For now, we focus only on comparing the proposed G-LMSM with the LMSM when equipped with a VGG4 backbone learned from random initialization. As shown, LMSM can outperform nonorth. matching, indicating that the PCA input subspace representation is effective in this task.
Next, "PCA+G-LMSM+softmax" can outperform the LMSM, indicating that the learning on the Grassmannian yields better accuracy than learning on Euclidean space. The results are evidence that both the orthogonality constraint and the manifold optimization seem to play an important role in learning a subspace.

\subsubsection{Performance across parameter settings}

\begin{figure}[tb]
\centering
\includegraphics[width= 6 cm]{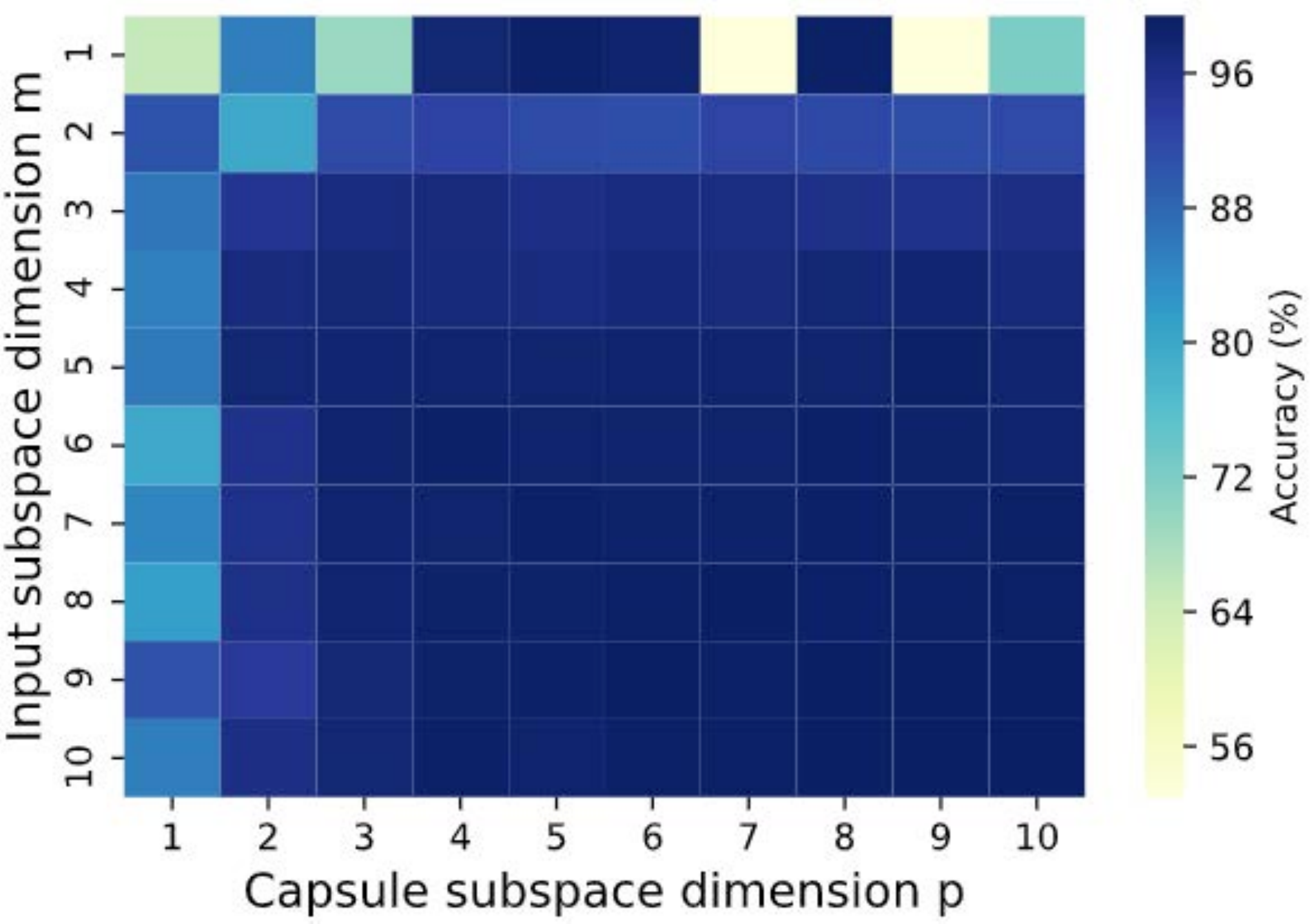}
\caption{Performance of G-LMSM across parameter settings, varying the dimension of capsule subspaces $p$ and input subspaces $m$. }
\label{fig:param_behavior}
\end{figure}
Figure~\ref{fig:param_behavior} shows the accuracies of G-LMSM+softmax when we vary the dimensions of subspaces, which are hyperparameters in our model. Both the dimension of reference subspaces $p$ and the dimension of input subspace $m$ are separately varied from $1$ to $10$. It can be inferred from the figure that G-LMSM+softmax has a lower performance from $1$ to around $3$ in both parameters because such low dimension subspaces are not enough to model the complexity of entities such as hand shapes, containing multiple poses, sizes, deformations and each subject's individual hand characteristics. When the dimension is $1$, the G-LMSM+softmax can be regarded as a CapPro, which as seen on the figure, yields a poor performance. Increasing the dimension of the input subspace improves the performance until about $m=8$ and $p=7$. It can be expected that a high value of $p$ contributes to overfitting. However, overall, it appears that the performance is not so sensitive to the value of hyperparameters if they are selected in a reasonable range, such as $5$ to $10$.

\subsubsection{Comparison against multiple baselines}

We compare G-LMSM to multiple baselines that do not contain its key features to evaluate its effectiveness. 

\textbf{Baselines}: There are five baselines in this experiment: (1) the mutual subspace method (MSM)~\cite{fukui2015difference}, a fundamental method of classifying subspaces with canonical angles. (2) "set average", a very simple idea to test the effectiveness of the backbone network. The set images are processed through the backbone VGG4, yielding a matrix of output features $\bm H \in \mathbb{R}^{512 \times 35}$, where the backbone features are of size $512$ and the number of images is $35$. Then the set images are averaged, which can be regarded as a setwise global average pooling, resulting in a single $512$ vector $\bm h$, which is then processed through a FC layer and softmax. (3) Softmax average, which processes each image through a FC layer and softmax and then averages the output probabilities. (4) Bilinear~\cite{gao2016compact}, which uses the features correlation matrix $\bm H \bm H\tp $ as a single integrated feature, and classifies it through an FC layer. And (5) CapPro average, where the average vector $\bm h$ is used as input to a CapPro layer.

\textbf{Results and discussion}:
Table~\ref{tab:exp_hand} shows the results. The first results, with no backbone convolutional layers, suggest two points: G-LMSM-based models can classify hand shapes more accurately than MSM, even when it is used on its own, i.e., just the G-LMSM layer as a standalone classifier; and the G-LMSM acting as a feature extractor followed by an FC classifier offers a much better performance. Neither the square root activation nor the repulsion loss seems to offer a gain of accuracy in this case.

The second part of Table~\ref{tab:exp_hand} shows methods that perform end-to-end learning with a VGG4 backbone. These results demonstrate two points: (1) the average feature vector of a set is a naive approach to treating image sets and that more information can be modeled from the features distribution; and (2) interestingly, the G-LMSM in a deep backbone network seems to perform better as a classifier than as a feature extractor. The reason might be that either the FC layer followed G-LMSM is prone to overfitting or that, in this kind of architecture, subspaces are better at representing classes directly rather than more abstract entities.

\subsection{Experiments on Emotion Recognition}

\begin{table}[htbp]
  \centering
  \caption{Results of the experiment on the AFEW dataset.}
    \begin{tabular}{lr}
    \toprule
    Method & \multicolumn{1}{l}{Accuracy (\%)} \\
    \midrule
    STM-ExpLet & 31.73 \\
    RSR-SPDML & 30.12 \\
    DCC   & 25.78 \\
    GDA   & 29.11 \\
    GGDA  & 29.45 \\
    PML   & 28.98 \\
    DeepO2P & 28.54 \\
    SPDNet & 34.23 \\
    GrNet & 34.23 \\
    G-LMSM+FC & 38.27 \\
    \bottomrule
    \end{tabular}%
  \label{tab:exp_afew}%
\end{table}%

\begin{table*}[ht]
\centering
  \caption{Results of the experiment on the YouTube Celebrities dataset.}
\begin{tabular}{llll} 
\hline
                           & Type                                            & Method                & Accuracy (\%)  \\ 
\hline
\multirow{8}{*}{Baselines} & \multirow{4}{*}{Classic}                        & DCC                   & 51.42 $\pm$ 4.95     \\
                           &                                                 & MMD                   & 54.04 $\pm$ 3.69     \\
                           &                                                 & CHISD                 & 60.42 $\pm$ 5.95     \\
                           &                                                 & PLRC                  & 61.28 $\pm$ 6.37     \\ 
\cline{2-4}
                           & \multirow{4}{*}{Deep}                           & DRM                   & 66.45 $\pm$ 5.07     \\
                           &                                                 & Resnet50 vote         & 64.18 $\pm$ 2.20     \\
                           &                                                 & Resnet18 set average  & 68.24 $\pm$ 3.32     \\
                           &                                                 & Resnet18 bilinear     & 67.85 $\pm$ 3.35     \\ 
\hline
\multirow{6}{*}{Proposed}  & Features (from ResNet50)                        & G-LMSM+FC              & 66.17 $\pm$ 4.19     \\ 
\cline{2-4}
                           & \multirow{5}{*}{End-to-end (backbone ResNet18)} & PCA+G-LMSM+softmax     & 68.80 $\pm$ 3.30     \\
                           &                                                 & AC+G-LMSM+softmax      & 68.93 $\pm$ 3.97     \\
                           &                                                 & AC+G-LMSM+sqrt         & 71.09 $\pm$ 3.62     \\
                           &                                                 & AC+G-LMSM (repulsion)    & 70.01 $\pm$ 2.87     \\
                           &                                                 & AC+G-LMSM+temp-softmax & 69.92 $\pm$ 3.79     \\
\hline
\end{tabular}
\label{tab:exp_ytc}
\end{table*}

We conducted an experiment on the task of emotion recognition to demonstrate the effectiveness of the proposed G-LMSM against various manifold and subspace-based methods that have been used in this task.

\textbf{Dataset}: We utilize the Acted Facial Expression in Wild (AFEW)~\cite{dhall2014emotion} dataset. The dataset contains $1,345$ sequences of $7$ types of facial expressions acted by $330$ actors in close to real-world settings.

\textbf{Settings}:
We follow the experiment protocol established by~\cite{liu2014learning,huang2017riemannian} to present the results on the validation set. The training videos are split into $1747$ small subvideos augmenting the numbers of sets, as established by~\cite{huang_building_2016}. For the evaluation, each facial frame is normalized to an image of size $20 \times 20$. For representation, following various works~\cite{liu2013partial,liu2014learning,huang2018building}, we represent the sequences of facial expressions with linear subspaces of dimension $10$, which exist on a Grassmann manifold $\mathbb{G}(400, 10)$.

\textbf{Baselines}: As the G-LMSM is a network layer that learns subspaces based on iterative Riemannian optimization, we compare it against the following methods: (1) a regular CNN that can handle image sets, namely the Deep Second-order Pooling (DeepO2P)~\cite{ionescu2015training}. (2) Methods based on subspace without Riemannian optimization: DCC~\cite{kim2007discriminative}, Grassmann Discriminant Analysis (GDA~\cite{hamm2008grassmann} and Grassmannian Graph-Embedding Discriminant Analysis (GGDA)~\cite{hamm2009extended}. And (3) methods based on Riemannian optimization: Projection Metric Learning (PML)~\cite{huang2015projection}, Expressionlets on Spatio-Temporal Manifold (STMExpLet)~\cite{liu2014learning}, Riemannian Sparse Representation combining with Manifold Learning on the manifold of SPD matrices (RSR-SPDML)~\cite{harandi2014manifold}, Network on SPD manifolds (SPDNet)~\cite{huang2017riemannian} and Grassmann net (GrNet)~\cite{huang2018building}. Especially, GrNet proposes a block of manifold layers for subspace data. GrNet-1 denotes the architecture with 1 block and GrNet-2 with 2 blocks.

\textbf{Results and discussion}:
The results can be seen in Table~\ref{tab:exp_afew}. The proposed method achieved quite competitive results compared to the manifold-based methods, by mapping the subspace into a vector and then uses simple Euclidean operations such as fully-connected layers and cross-entropy loss, in contrast to several methods that use complex approaches requiring multiple matrix decompositions.
First, G-LMSM+softmax outperforms popular methods such as GDA and GGDA, which have a similar purpose of mapping Grassmann manifold data into a vector representation. The reason is perhaps that G-LMSM learns both the representation and discrimination in an end-to-end manner, while GDA uses a kernel function to represent subspaces and learns the discriminant independently.
Methods such as SPDNet and GrNet are composed of many complex layers involving SVD, QR decompositions, and Gram–Schmidt orthogonalization and its derivatives are utilized as well. They increase in complexity as the number of layers increases by repeating these operations, which are not easily scalable to use in GPUs. On the other hand, the proposed method provides competitive results with fewer layers and no decompositions, making it naturally parallelizable and scalable.

\subsection{Experiment on Face Identification}

We conducted an experiment of face identification with the YouTube Celebrities (YTC) dataset.
\textbf{Dataset}: The YTC dataset~\cite{ytc} contains $1910$ videos of $47$ identities.
Similarly to \cite{reconst}, as an image set, we used a set of face images extracted from a video by the Incremental Learning Tracker~\cite{ilt}.
All the extracted face images were scaled to 30 $\times$ 30 pixels and converted to grayscale.

\textbf{Settings}: Three videos per each person were randomly selected as training data, and six videos per each person were randomly selected as test data.
We repeated the above procedure five times and measured the average accuracy.

\textbf{Baselines}:
(1) Classic image set-based methods: Discriminative Canonical Correlations (DCC)~\cite{kim2007discriminative}, manifold-manifold distance (MMD)~\cite{wang2008manifold}, Convex Hull-based Image Set Distance (CHISD)~\cite{cevikalp2010face}, Pairwise linear regression classification (PLRC)~\cite{feng2016pairwise}. (2) Deep learning methods: deep reconstruction model (DRM)~\cite{shah2017efficient};
Resnet vote~\cite{resnet} is a baseline consisting of a Resnet50~\cite{resnet} fine-tuned to this dataset with the cross-entropy loss in a single image setting. For the fine-tuning, we added two fully connected (FC) layers after the last global average pooling layer in the network. The first FC layer outputs a 1024 dimension vector through the ReLU~\cite{relu} function, and the second layer outputs a 47 (the number of classes) dimension vector through the softmax function. Hyperparameters of the optimizer were used as suggested by the original paper. For classification, each image of a set is classified independently and a majority voting strategy of all predictions is used to select a single class prediction for the whole set. The model called Resnet18 set average is the same as used in the hand shape experiment. The Resnet18 bilinear model uses the correlation matrix of the backbone features of the image set, generating a $512 \times 512$ matrix. The vectorized matrix is processed through an FC layer and softmax for classification.

\textbf{Results and discussion}:
Table~\ref{tab:exp_ytc} shows the results. G-LMSM outperforms not only the classical manifold methods, but also various deep methods. Although G-LMSM uses the same backbone as some methods, it still can achieve better results. In the case of Resnet50 vote and the G-LMSM+FC, the key difference is that voting is a simple heuristic approach to image set recognition, as the model is not aware of the set as a whole. This leads to unstable and diverging predictions on a set's class, as it does not take into account the underlying set distribution. The voting method is not differentiable by nature, so it is not trivial to extend this idea to an end-to-end approach. Then, set average seems again to be a naive approach to treating image sets and that more information can be modeled from the features distribution. The bilinear model could in theory capture the underlying set information from the correlations, but the high dimensionality of the correlation matrix makes it almost impossible to learn a FC classifier without overfitting. 

These results confirm the advantage of G-LMSM over these methods using the same backbone: mutual subspaces provide a low-dimensional, differentiable approach to modeling the set distribution, ultimately leading to a straightforward integration of the deep convolutional networks and the problem of image sets.

Furthermore, the AC+G-LMSM+softmax seems to yield the same result as PCA+G-LMSM+softmax, evidence that it is a good approximation for processing the input data. Then, the proposed variations of G-LMSM, namely, repulsion loss, the square root activation and the temperature softmax (temp-softmax) appear to provide an advantage over the base G-LMSM.  
``	
\section{Conclusion}

In this paper, we addressed the problem of image set recognition. We proposed two subspace-based methods, named learning mutual subspace method (LMSM) and Grassmannian learning mutual subspace method (G-LMSM), that can recognize image sets with DNN as feature extractors in an end-to-end fashion. The proposed methods generalize the classic average learning subspace method (ALSM), in that LMSM and G-LMSM can handle an image set as input and perform subspace matching. The key idea of G-LMSM is to learn the subspaces with Riemannian stochastic gradient descent on the Grassmann manifold, ensuring each parameter is always an orthogonal basis of a subspace.

Extensive experiments on hand shape recognition, face identification, and facial emotion recognition showed that LMSM and G-LMSM have a high discriminant ability in image set recognition. It can be used easily as a standalone method or within larger DNN frameworks. The experiments also reveal that Riemannian optimization is effective in learning G-LMSM within neural networks.


%
\IEEEpeerreviewmaketitle

\ifCLASSOPTIONcaptionsoff
  \newpage
\fi



%



\bibliography{main.bib}{}

\begin{thebibliography}{10}
\providecommand{\url}[1]{#1}
\csname url@samestyle\endcsname
\providecommand{\newblock}{\relax}
\providecommand{\bibinfo}[2]{#2}
\providecommand{\BIBentrySTDinterwordspacing}{\spaceskip=0pt\relax}
\providecommand{\BIBentryALTinterwordstretchfactor}{4}
\providecommand{\BIBentryALTinterwordspacing}{\spaceskip=\fontdimen2\font plus
\BIBentryALTinterwordstretchfactor\fontdimen3\font minus
  \fontdimen4\font\relax}
\providecommand{\BIBforeignlanguage}[2]{{%
\expandafter\ifx\csname l@#1\endcsname\relax
\typeout{** WARNING: IEEEtran.bst: No hyphenation pattern has been}%
\typeout{** loaded for the language `#1'. Using the pattern for}%
\typeout{** the default language instead.}%
\else
\language=\csname l@#1\endcsname
\fi
#2}}
\providecommand{\BIBdecl}{\relax}
\BIBdecl

\bibitem{shashua_photometric_1997}
\BIBentryALTinterwordspacing
A.~Shashua, ``\BIBforeignlanguage{en}{On {Photometric} {Issues} in {3D}
  {Visual} {Recognition} from a {Single} {2D} {Image}},''
  \emph{\BIBforeignlanguage{en}{International Journal of Computer Vision}},
  vol.~21, no.~1, pp. 99--122, Jan. 1997. [Online]. Available:
  \url{https://doi.org/10.1023/A:1007975506780}
\BIBentrySTDinterwordspacing

\bibitem{belhumeur_what_1998}
\BIBentryALTinterwordspacing
P.~N. Belhumeur and D.~J. Kriegman, ``\BIBforeignlanguage{en}{What {Is} the
  {Set} of {Images} of an {Object} {Under} {All} {Possible} {Illumination}
  {Conditions}?}'' \emph{\BIBforeignlanguage{en}{International Journal of
  Computer Vision}}, vol.~28, no.~3, pp. 245--260, Jul. 1998. [Online].
  Available: \url{https://doi.org/10.1023/A:1008005721484}
\BIBentrySTDinterwordspacing

\bibitem{lee2005acquiring}
K.-C. Lee, J.~Ho, and D.~J. Kriegman, ``Acquiring linear subspaces for face
  recognition under variable lighting,'' \emph{IEEE Transactions on pattern
  analysis and machine intelligence}, vol.~27, no.~5, pp. 684--698, 2005.

\bibitem{basri2003lambertian}
R.~Basri and D.~W. Jacobs, ``Lambertian reflectance and linear subspaces,''
  \emph{IEEE transactions on pattern analysis and machine intelligence},
  vol.~25, no.~2, pp. 218--233, 2003.

\bibitem{kim2007discriminative}
T.-K. Kim, J.~Kittler, and R.~Cipolla, ``Discriminative learning and
  recognition of image set classes using canonical correlations,'' \emph{IEEE
  Transactions on Pattern Analysis and Machine Intelligence}, vol.~29, no.~6,
  pp. 1005--1018, 2007.

\bibitem{sogi2022constrained}
N.~Sogi, R.~Zhu, J.-H. Xue, and K.~Fukui, ``Constrained mutual convex cone
  method for image set based recognition,'' \emph{Pattern Recognition}, vol.
  121, p. 108190, 2022.

\bibitem{souza2020enhanced}
L.~S. Souza, B.~B. Gatto, J.-H. Xue, and K.~Fukui, ``Enhanced grassmann
  discriminant analysis with randomized time warping for motion recognition,''
  \emph{Pattern Recognition}, vol.~97, p. 107028, 2020.

\bibitem{watanabe1973subspace}
S.~Watanabe and N.~Pakvasa, ``Subspace method of pattern recognition,'' in
  \emph{Proc. 1st. IJCPR}, 1973, pp. 25--32.

\bibitem{iijima1974theory}
T.~Iijima, H.~Genchi, and K.-i. Mori, ``A theory of character recognition by
  pattern matching method,'' in \emph{Learning systems and intelligent
  robots}.\hskip 1em plus 0.5em minus 0.4em\relax Springer, 1974, pp. 437--450.

\bibitem{oja1983subspace}
E.~Oja, \emph{Subspace methods of pattern recognition}.\hskip 1em plus 0.5em
  minus 0.4em\relax Research Studies Press, 1983, vol.~6.

\bibitem{georghiades2001few}
A.~S. Georghiades, P.~N. Belhumeur, and D.~J. Kriegman, ``From few to many:
  Illumination cone models for face recognition under variable lighting and
  pose,'' \emph{IEEE transactions on pattern analysis and machine
  intelligence}, vol.~23, no.~6, pp. 643--660, 2001.

\bibitem{hotelling1992relations}
H.~Hotelling, ``Relations between two sets of variates,'' in
  \emph{Breakthroughs in statistics}.\hskip 1em plus 0.5em minus 0.4em\relax
  Springer, 1992, pp. 162--190.

\bibitem{afriat1957orthogonal}
S.~N. Afriat, ``Orthogonal and oblique projectors and the characteristics of
  pairs of vector spaces,'' in \emph{Mathematical Proceedings of the Cambridge
  Philosophical Society}, vol.~53, no.~4.\hskip 1em plus 0.5em minus
  0.4em\relax Cambridge University Press, 1957, pp. 800--816.

\bibitem{fukui2015difference}
K.~Fukui and A.~Maki, ``Difference subspace and its generalization for
  subspace-based methods,'' \emph{Pattern Analysis and Machine Intelligence,
  IEEE Transactions on}, vol.~37, no.~11, pp. 2164--2177, 2015.

\bibitem{yamaguchi2003smartface}
O.~Yamaguchi and K.~Fukui, ``Smartface--a robust face recognition system under
  varying facial pose and expression,'' 2003.

\bibitem{fukui2007kernel}
K.~Fukui and O.~Yamaguchi, ``The kernel orthogonal mutual subspace method and
  its application to 3d object recognition,'' in \emph{Asian Conference on
  Computer Vision}.\hskip 1em plus 0.5em minus 0.4em\relax Springer, 2007, pp.
  467--476.

\bibitem{sakano2000kernel}
H.~Sakano and N.~Mukawa, ``Kernel mutual subspace method for robust facial
  image recognition,'' in \emph{KES'2000. Fourth International Conference on
  Knowledge-Based Intelligent Engineering Systems and Allied Technologies.
  Proceedings (Cat. No. 00TH8516)}, vol.~1.\hskip 1em plus 0.5em minus
  0.4em\relax IEEE, 2000, pp. 245--248.

\bibitem{hamm2008grassmann}
J.~Hamm and D.~D. Lee, ``Grassmann discriminant analysis: a unifying view on
  subspace-based learning,'' in \emph{Proceedings of the 25th international
  conference on Machine learning}.\hskip 1em plus 0.5em minus 0.4em\relax ACM,
  2008, pp. 376--383.

\bibitem{hamm2009extended}
------, ``Extended grassmann kernels for subspace-based learning,'' in
  \emph{Advances in neural information processing systems}, 2009, pp. 601--608.

\bibitem{kohonen1979spectral}
T.~Kohonen, G.~N{\'e}meth, K.-J. Bry, M.~Jalanko, and H.~Riittinen, ``Spectral
  classification of phonemes by learning subspaces,'' in \emph{ICASSP'79. IEEE
  International Conference on Acoustics, Speech, and Signal Processing},
  vol.~4.\hskip 1em plus 0.5em minus 0.4em\relax IEEE, 1979, pp. 97--100.

\bibitem{oja_alsm_1983}
E.~Oja and M.~Kuusela, ``The {ALSM} algorithm — an improved subspace method
  of classification,'' \emph{Pattern Recognition}, vol.~16, no.~4, pp.
  421--427, 1983.

\bibitem{absil2004riemannian}
P.-A. Absil, R.~Mahony, and R.~Sepulchre, ``Riemannian geometry of grassmann
  manifolds with a view on algorithmic computation,'' \emph{Acta Applicandae
  Mathematica}, vol.~80, no.~2, pp. 199--220, 2004.

\bibitem{borisenko1991grassmann}
A.~A. Borisenko and Y.~A. Nikolaevskii, ``Grassmann manifolds and the grassmann
  image of submanifolds,'' \emph{Russian mathematical surveys}, vol.~46, no.~2,
  pp. 45--94, 1991.

\bibitem{fujii2002introduction}
K.~Fujii, ``Introduction to grassmann manifolds and quantum computation,''
  \emph{Journal of Applied Mathematics}, vol.~2, no.~8, pp. 371--405, 2002.

\bibitem{ganea2018riemannian}
O.-E. Ganea and G.~B{\'e}cigneul, ``Riemannian adaptive optimization methods,''
  in \emph{7th International Conference on Learning Representations (ICLR
  2019)}, 2018.

\bibitem{bonnabel2013stochastic}
S.~Bonnabel, ``Stochastic gradient descent on riemannian manifolds,''
  \emph{IEEE Transactions on Automatic Control}, vol.~58, no.~9, pp.
  2217--2229, 2013.

\bibitem{wang2018settoset}
\BIBentryALTinterwordspacing
L.~Wang, H.~Cheng, and Z.~Liu, ``\BIBforeignlanguage{en}{A set-to-set nearest
  neighbor approach for robust and efficient face recognition with image
  sets},'' \emph{\BIBforeignlanguage{en}{Journal of Visual Communication and
  Image Representation}}, vol.~53, pp. 13--19, May 2018. [Online]. Available:
  \url{https://www.sciencedirect.com/science/article/pii/S1047320318300300}
\BIBentrySTDinterwordspacing

\bibitem{wang2017discriminative}
\BIBentryALTinterwordspacing
W.~Wang, R.~Wang, S.~Shan, and X.~Chen,
  ``\BIBforeignlanguage{en}{Discriminative {Covariance} {Oriented}
  {Representation} {Learning} for {Face} {Recognition} with {Image} {Sets}},''
  in \emph{\BIBforeignlanguage{en}{2017 {IEEE} {Conference} on {Computer}
  {Vision} and {Pattern} {Recognition} ({CVPR})}}.\hskip 1em plus 0.5em minus
  0.4em\relax Honolulu, HI: IEEE, Jul. 2017, pp. 5749--5758. [Online].
  Available: \url{http://ieeexplore.ieee.org/document/8100092/}
\BIBentrySTDinterwordspacing

\bibitem{lu2015multi}
J.~Lu, G.~Wang, W.~Deng, P.~Moulin, and J.~Zhou, ``Multi-manifold deep metric
  learning for image set classification,'' in \emph{Proceedings of the IEEE
  conference on computer vision and pattern recognition}, 2015, pp. 1137--1145.

\bibitem{feng2016pairwise}
Q.~Feng, Y.~Zhou, and R.~Lan, ``Pairwise {Linear} {Regression} {Classification}
  for {Image} {Set} {Retrieval},'' in \emph{2016 {IEEE} {Conference} on
  {Computer} {Vision} and {Pattern} {Recognition} ({CVPR})}, Jun. 2016, pp.
  4865--4872, iSSN: 1063-6919.

\bibitem{naseem2010linear}
I.~Naseem, R.~Togneri, and M.~Bennamoun, ``Linear {Regression} for {Face}
  {Recognition},'' \emph{IEEE Transactions on Pattern Analysis and Machine
  Intelligence}, vol.~32, no.~11, pp. 2106--2112, Nov. 2010, conference Name:
  IEEE Transactions on Pattern Analysis and Machine Intelligence.

\bibitem{stallkamp2007video-based}
J.~Stallkamp, H.~K. Ekenel, and R.~Stiefelhagen, ``Video-based {Face}
  {Recognition} on {Real}-{World} {Data},'' in \emph{2007 {IEEE} 11th
  {International} {Conference} on {Computer} {Vision}}, Oct. 2007, pp. 1--8,
  iSSN: 2380-7504.

\bibitem{zhou2002probabilistic}
S.~Zhou and R.~Chellappa, ``\BIBforeignlanguage{en}{Probabilistic {Human}
  {Recognition} from {Video}},'' in \emph{\BIBforeignlanguage{en}{Computer
  {Vision} — {ECCV} 2002}}, ser. Lecture {Notes} in {Computer} {Science},
  A.~Heyden, G.~Sparr, M.~Nielsen, and P.~Johansen, Eds.\hskip 1em plus 0.5em
  minus 0.4em\relax Berlin, Heidelberg: Springer, 2002, pp. 681--697.

\bibitem{uijlings2013selective}
J.~R. Uijlings, K.~E. Van De~Sande, T.~Gevers, and A.~W. Smeulders, ``Selective
  search for object recognition,'' \emph{International journal of computer
  vision}, vol. 104, no.~2, pp. 154--171, 2013.

\bibitem{li2009boosting}
X.~Li, K.~Fukui, and N.~Zheng, ``Boosting constrained mutual subspace method
  for robust image-set based object recognition,'' in \emph{Twenty-First
  International Joint Conference on Artificial Intelligence}, 2009.

\bibitem{jarrett2009best}
K.~Jarrett, K.~Kavukcuoglu, M.~Ranzato, and Y.~LeCun, ``What is the best
  multi-stage architecture for object recognition?'' in \emph{2009 IEEE 12th
  international conference on computer vision}.\hskip 1em plus 0.5em minus
  0.4em\relax IEEE, 2009, pp. 2146--2153.

\bibitem{zhao2019review}
Z.-Q. Zhao, S.-T. Xu, D.~Liu, W.-D. Tian, and Z.-D. Jiang, ``A review of image
  set classification,'' \emph{Neurocomputing}, vol. 335, pp. 251--260, 2019.

\bibitem{arandjelovic2005face}
O.~Arandjelovic, G.~Shakhnarovich, J.~Fisher, R.~Cipolla, and T.~Darrell,
  ``Face recognition with image sets using manifold density divergence,'' in
  \emph{2005 IEEE Computer Society Conference on Computer Vision and Pattern
  Recognition (CVPR'05)}, vol.~1.\hskip 1em plus 0.5em minus 0.4em\relax IEEE,
  2005, pp. 581--588.

\bibitem{maeda2004towards}
K.-i. Maeda, O.~Yamaguchi, and K.~Fukui, ``Towards 3-dimensional pattern
  recognition,'' in \emph{Joint IAPR International Workshops on Statistical
  Techniques in Pattern Recognition (SPR) and Structural and Syntactic Pattern
  Recognition (SSPR)}.\hskip 1em plus 0.5em minus 0.4em\relax Springer, 2004,
  pp. 1061--1068.

\bibitem{kim2009canonical}
T.-K. Kim and R.~Cipolla, ``Canonical correlation analysis of video volume
  tensors for action categorization and detection,'' \emph{IEEE Transactions on
  Pattern Analysis and Machine Intelligence}, vol.~31, no.~8, pp. 1415--1428,
  2009.

\bibitem{wang2012manifold}
R.~Wang, S.~Shan, X.~Chen, Q.~Dai, and W.~Gao, ``Manifold--manifold distance
  and its application to face recognition with image sets,'' \emph{IEEE
  Transactions on Image Processing}, vol.~21, no.~10, pp. 4466--4479, 2012.

\bibitem{wang2009manifold}
R.~Wang and X.~Chen, ``Manifold discriminant analysis,'' in \emph{2009 IEEE
  Conference on Computer Vision and Pattern Recognition}.\hskip 1em plus 0.5em
  minus 0.4em\relax IEEE, 2009, pp. 429--436.

\bibitem{cevikalp2010face}
H.~Cevikalp and B.~Triggs, ``Face recognition based on image sets,'' in
  \emph{2010 IEEE Computer Society Conference on Computer Vision and Pattern
  Recognition}.\hskip 1em plus 0.5em minus 0.4em\relax IEEE, 2010, pp.
  2567--2573.

\bibitem{hu2012face}
Y.~Hu, A.~S. Mian, and R.~Owens, ``Face recognition using sparse approximated
  nearest points between image sets,'' \emph{IEEE transactions on pattern
  analysis and machine intelligence}, vol.~34, no.~10, pp. 1992--2004, 2012.

\bibitem{hu2011sparse}
------, ``Sparse approximated nearest points for image set classification,'' in
  \emph{CVPR 2011}.\hskip 1em plus 0.5em minus 0.4em\relax IEEE, 2011, pp.
  121--128.

\bibitem{harandi2013dictionary}
M.~Harandi, C.~Sanderson, C.~Shen, and B.~C. Lovell, ``Dictionary learning and
  sparse coding on grassmann manifolds: An extrinsic solution,'' in
  \emph{Proceedings of the IEEE international conference on computer vision},
  2013, pp. 3120--3127.

\bibitem{liao_face_2019}
M.~Liao and X.~Gu, ``Face recognition based on dictionary learning and subspace
  learning,'' \emph{Digital Signal Processing}, vol.~90, pp. 110--124, 2019.

\bibitem{hamm2008thesis}
J.~Hamm, ``Subspace-based learning with grassmann kernels,'' 2008.

\bibitem{harandi2011graph}
M.~T. Harandi, C.~Sanderson, S.~Shirazi, and B.~C. Lovell, ``Graph embedding
  discriminant analysis on {G}rassmannian manifolds for improved image set
  matching,'' in \emph{Computer Vision and Pattern Recognition, IEEE Conference
  on}.\hskip 1em plus 0.5em minus 0.4em\relax IEEE, 2011, pp. 2705--2712.

\bibitem{lincon2016egda}
L.~Souza, H.~Hino, and K.~Fukui, ``3d object recognition with enhanced
  grassmann discriminant analysis,'' in \emph{ACCV 2016 Workshop (HIS 2016)},
  2016.

\bibitem{zhu_towards_2018}
P.~Zhu, H.~Cheng, Q.~Hu, Q.~Wang, and C.~Zhang, ``Towards {Generalized} and
  {Efficient} {Metric} {Learning} on {Riemannian} {Manifold},''
  \emph{Proceedings of the Twenty-Seventh International Joint Conference on
  Artificial Intelligence}, pp. 3235--3241, 2018.

\bibitem{sogi2020metric}
N.~Sogi, L.~S. Souza, B.~B. Gatto, and K.~Fukui, ``Metric learning with a-based
  scalar product for image-set recognition,'' in \emph{Proceedings of the
  IEEE/CVF Conference on Computer Vision and Pattern Recognition Workshops},
  2020, pp. 850--851.

\bibitem{luo_robust_2019}
L.~Luo, J.~Xu, C.~Deng, and H.~Huang, ``Robust {Metric} {Learning} on
  {Grassmann} {Manifolds} with {Generalization} {Guarantees},''
  \emph{Proceedings of the AAAI Conference on Artificial Intelligence},
  vol.~33, pp. 4480--4487, 2019.

\bibitem{zhang2018cappronet}
L.~Zhang, M.~Edraki, and G.-J. Qi, ``Cappronet: Deep feature learning via
  orthogonal projections onto capsule subspaces,'' in \emph{Advances in neural
  information processing systems}, 2018.

\bibitem{hinton2011transforming}
G.~E. Hinton, A.~Krizhevsky, and S.~D. Wang, ``Transforming auto-encoders,'' in
  \emph{International conference on artificial neural networks}.\hskip 1em plus
  0.5em minus 0.4em\relax Springer, 2011, pp. 44--51.

\bibitem{sabour2017dynamic}
S.~Sabour, N.~Frosst, and G.~E. Hinton, ``Dynamic routing between capsules,''
  \emph{arXiv preprint arXiv:1710.09829}, 2017.

\bibitem{watanabe1967evaluation}
S.~Watanabe, ``Evaluation and selection of variables in pattern recognition,''
  \emph{Computer and Information Science II}, pp. 91--122, 1967.

\bibitem{iijima1972theoretical}
T.~Iijima, ``A theoretical study of pattern recognition by matching method,''
  in \emph{Proc. of First USA-Japan Computer Conf., 1972}, 1972, pp. 42--48.

\bibitem{edelman1998geometry}
A.~Edelman, T.~A. Arias, and S.~T. Smith, ``The geometry of algorithms with
  orthogonality constraints,'' \emph{SIAM journal on Matrix Analysis and
  Applications}, vol.~20, no.~2, pp. 303--353, 1998.

\bibitem{becigneul2018riemannian}
\BIBentryALTinterwordspacing
G.~B{\'e}cigneul and O.-E. Ganea, ``Riemannian adaptive optimization methods,''
  \emph{arXiv preprint arXiv:1810.00760}, 2018. [Online]. Available:
  \url{https://openreview.net/forum?id=r1eiqi09K7}
\BIBentrySTDinterwordspacing

\bibitem{yang2018square}
X.~Yang, Y.~Chen, and H.~Liang, ``Square root based activation function in
  neural networks,'' in \emph{2018 International Conference on Audio, Language
  and Image Processing (ICALIP)}.\hskip 1em plus 0.5em minus 0.4em\relax IEEE,
  2018, pp. 84--89.

\bibitem{guo2017calibration}
C.~Guo, G.~Pleiss, Y.~Sun, and K.~Q. Weinberger, ``On calibration of modern
  neural networks,'' in \emph{International Conference on Machine
  Learning}.\hskip 1em plus 0.5em minus 0.4em\relax PMLR, 2017, pp. 1321--1330.

\bibitem{hinton2015distilling}
G.~Hinton, O.~Vinyals, and J.~Dean, ``Distilling the knowledge in a neural
  network,'' \emph{arXiv preprint arXiv:1503.02531}, 2015.

\bibitem{liang2017enhancing}
S.~Liang, Y.~Li, and R.~Srikant, ``Enhancing the reliability of
  out-of-distribution image detection in neural networks,'' \emph{arXiv
  preprint arXiv:1706.02690}, 2017.

\bibitem{agarwala2020temperature}
A.~Agarwala, J.~Pennington, Y.~Dauphin, and S.~Schoenholz, ``Temperature check:
  theory and practice for training models with softmax-cross-entropy losses,''
  \emph{arXiv preprint arXiv:2010.07344}, 2020.

\bibitem{gao2016compact}
Y.~Gao, O.~Beijbom, N.~Zhang, and T.~Darrell, ``Compact bilinear pooling,'' in
  \emph{Proceedings of the IEEE conference on computer vision and pattern
  recognition}, 2016, pp. 317--326.

\bibitem{dhall2014emotion}
A.~Dhall, R.~Goecke, J.~Joshi, K.~Sikka, and T.~Gedeon, ``Emotion recognition
  in the wild challenge 2014: Baseline, data and protocol,'' in
  \emph{Proceedings of the 16th international conference on multimodal
  interaction}.\hskip 1em plus 0.5em minus 0.4em\relax ACM, 2014, pp. 461--466.

\bibitem{liu2014learning}
M.~Liu, S.~Shan, R.~Wang, and X.~Chen, ``Learning expressionlets on
  spatio-temporal manifold for dynamic facial expression recognition,'' in
  \emph{Proceedings of the IEEE Conference on Computer Vision and Pattern
  Recognition}, 2014, pp. 1749--1756.

\bibitem{huang2017riemannian}
Z.~Huang and L.~Van~Gool, ``A riemannian network for spd matrix learning,'' in
  \emph{Thirty-First AAAI Conference on Artificial Intelligence}, 2017.

\bibitem{huang_building_2016}
Z.~Huang, J.~Wu, and L.~V. Gool, ``Building {Deep} {Networks} on {Grassmann}
  {Manifolds},'' \emph{arXiv}, 2016.

\bibitem{liu2013partial}
M.~Liu, R.~Wang, Z.~Huang, S.~Shan, and X.~Chen, ``Partial least squares
  regression on grassmannian manifold for emotion recognition,'' in
  \emph{Proceedings of the 15th ACM on International conference on multimodal
  interaction}.\hskip 1em plus 0.5em minus 0.4em\relax ACM, 2013, pp. 525--530.

\bibitem{huang2018building}
Z.~Huang, J.~Wu, and L.~Van~Gool, ``Building deep networks on grassmann
  manifolds,'' in \emph{Thirty-Second AAAI Conference on Artificial
  Intelligence}, 2018.

\bibitem{ionescu2015training}
C.~Ionescu, O.~Vantzos, and C.~Sminchisescu, ``Training deep networks with
  structured layers by matrix backpropagation,'' \emph{arXiv preprint
  arXiv:1509.07838}, 2015.

\bibitem{huang2015projection}
Z.~Huang, R.~Wang, S.~Shan, and X.~Chen, ``Projection metric learning on
  grassmann manifold with application to video based face recognition,'' in
  \emph{Proceedings of the IEEE conference on computer vision and pattern
  recognition}, 2015, pp. 140--149.

\bibitem{harandi2014manifold}
M.~T. Harandi, M.~Salzmann, and R.~Hartley, ``From manifold to manifold:
  Geometry-aware dimensionality reduction for spd matrices,'' in \emph{European
  conference on computer vision}.\hskip 1em plus 0.5em minus 0.4em\relax
  Springer, 2014, pp. 17--32.

\bibitem{ytc}
M.~Kim, S.~Kumar, V.~Pavlovic, and H.~Rowley, ``Face tracking and recognition
  with visual constraints in real-world videos,'' in \emph{Computer Vision and
  Pattern Recognition}.\hskip 1em plus 0.5em minus 0.4em\relax IEEE, 2008, pp.
  1--8.

\bibitem{reconst}
S.~A.~A. Shah, U.~Nadeem, M.~Bennamoun, F.~A. Sohel, and R.~Togneri,
  ``Efficient image set classification using linear regression based image
  reconstruction.'' in \emph{Computer Vision and Pattern Recognition
  Workshops}, 2017, pp. 601--610.

\bibitem{ilt}
D.~A. Ross, J.~Lim, R.-S. Lin, and M.-H. Yang, ``Incremental learning for
  robust visual tracking,'' \emph{International Journal of Computer Vision},
  vol.~77, no. 1-3, pp. 125--141, 2008.

\bibitem{wang2008manifold}
R.~Wang, S.~Shan, X.~Chen, and W.~Gao, ``Manifold-manifold distance with
  application to face recognition based on image set,'' in \emph{2008 IEEE
  Conference on Computer Vision and Pattern Recognition}.\hskip 1em plus 0.5em
  minus 0.4em\relax IEEE, 2008, pp. 1--8.

\bibitem{shah2017efficient}
S.~A. Shah, U.~Nadeem, M.~Bennamoun, F.~Sohel, and R.~Togneri, ``Efficient
  image set classification using linear regression based image
  reconstruction,'' in \emph{Proceedings of the IEEE Conference on Computer
  Vision and Pattern Recognition Workshops}, 2017, pp. 99--108.

\bibitem{resnet}
K.~He, X.~Zhang, S.~Ren, and J.~Sun, ``Deep residual learning for image
  recognition,'' in \emph{The IEEE Conference on Computer Vision and Pattern
  Recognition}, 2016.

\bibitem{relu}
V.~Nair and G.~E. Hinton, ``Rectified linear units improve restricted boltzmann
  machines,'' in \emph{Proceedings of the 27th International Conference on
  Machine Learning}, 2010, pp. 807--814.

\end{thebibliography}
\bibliographystyle{IEEEtran}

\end{document}